%% file: main.tex
\documentclass{article}

\usepackage[final,nonatbib]{neurips_2020}

\usepackage[utf8]{inputenc} %
\usepackage[T1]{fontenc}    %
\usepackage{hyperref}       %
\usepackage{url}            %
\usepackage{booktabs}       %
\usepackage{amsfonts}       %
\usepackage{nicefrac}       %
\usepackage{microtype}      %
\usepackage{comment}
\usepackage{amsmath}
\usepackage{amssymb}
\usepackage{cleveref}
\usepackage{graphicx}
\usepackage{thmtools, thm-restate}

\usepackage{wrapfig,lipsum,booktabs}
\newcommand{\mname}[0]{\text{Composite Score Matching}}

\input{cmd.tex}

\title{Autoregressive Score Matching}

\author{
  Chenlin Meng\\
  Stanford University\\
  \texttt{chenlin@stanford.edu}\\
  \And
  Lantao Yu\\
  Stanford University\\
  \texttt{lantaoyu@cs.stanford.edu} \\
  \And
  Yang Song\\
  Stanford University\\
  \texttt{yangsong@cs.stanford.edu}\\
  \And
  Jiaming Song\\
  Stanford University\\
  \texttt{tsong@cs.stanford.edu}\\
  \And
  Stefano Ermon\\
  Stanford University\\
  \texttt{ermon@cs.stanford.edu}\\
}

\begin{document}
\maketitle
\input{abstract}

\input{intro}

\input{background}

\input{method}

\input{exp.tex}

\input{exp_vae}

\input{related}
\input{conclusion}
\bibliography{bib}
\input{app}

\end{document}

%% file: cmd.tex
\usepackage[export]{adjustbox}
\usepackage{adjustbox, bigstrut, tabularx, multirow, makecell, diagbox}
\usepackage{algorithmicx,algpseudocode,algorithm}
\usepackage{amsmath,amsfonts,bm,amssymb,mathtools,amsthm}
\usepackage{color}
\usepackage[dvipsnames]{xcolor}
\usepackage{subcaption}
\usepackage{float}

\newcommand{\bs}[1]{\boldsymbol{#1}}
\newcommand{\mbb}[1]{\mathbb{#1}}

\newcommand{\mcal}{\mathcal}

\newtheorem{lemma}{Lemma}
\newtheorem{theorem}{Theorem}
\newtheorem{corollary}{Corollary}

\newtheorem{definition}{Definition}
\newtheorem{innercustomthm}{Theorem}
\newenvironment{customthm}[1]
{\renewcommand\theinnercustomthm{#1}\innercustomthm}
{\endinnercustomthm}

\newcommand{\be}{\begin{equation}}
	\newcommand{\ee}{\end{equation}}

\definecolor{Gray}{gray}{0.85}
\definecolor{LightCyan}{rgb}{0.88,1,1}

\makeatletter
\usepackage{xspace}
\def\@onedot{\ifx\@let@token.\else.\null\fi\xspace}
\DeclareRobustCommand\onedot{\futurelet\@let@token\@onedot}

\newcommand{\eqnref}[1]{Eq\onedot~(\plaineqref{#1})}
\newcommand{\figref}[1]{Fig\onedot~\ref{#1}}

\newcommand{\equref}[1]{Eq\onedot~\eqref{#1}}

\newcommand{\tabref}[1]{Tab\onedot~\ref{#1}}

\newcommand{\appref}[1]{Appendix~\ref{#1}}

\newcommand{\lemref}[1]{Lemma~\ref{#1}}

\newcommand{\bfx}{\mathbf{x}}
\newcommand{\bfc}{\mathbf{c}}

\newcommand{\bfz}{\mathbf{z}}

\newcommand{\bfe}{{\bs{\epsilon}}}

\newcommand{\bfh}{\mathbf{h}}
\def\eg{\emph{e.g}\onedot}

\def\ie{\emph{i.e}\onedot}

\usepackage{amsmath,amsfonts,bm,amssymb,mathtools,amsthm}
\usepackage{color,xcolor}

\def\figref#1{Figure~\ref{#1}}

\def\eqref#1{equation~\ref{#1}}

\def\plaineqref#1{\ref{#1}}

\def\1{\bm{1}}

\def\eps{{\epsilon}}

\DeclareMathAlphabet{\mathsfit}{\encodingdefault}{\sfdefault}{m}{sl}
\SetMathAlphabet{\mathsfit}{bold}{\encodingdefault}{\sfdefault}{bx}{n}

\definecolor{LightCyan}{rgb}{0.88,1,1}

%% file: abstract.tex
\begin{abstract}
Autoregressive models use chain rule to define a joint probability distribution as a product of conditionals. These conditionals need to be normalized, imposing constraints on the functional families that can be used. To increase flexibility, we propose autoregressive conditional score models (AR-CSM) where we parameterize the joint distribution in terms of the derivatives of univariate log-conditionals (scores), which need not be normalized.
To train AR-CSM, we introduce a new divergence between distributions named \mname{} (CSM).
For AR-CSM models, this divergence between data and model distributions can be computed and optimized  efficiently, requiring no expensive sampling or adversarial training. Compared to previous score matching algorithms, our method is more scalable to high dimensional data and more stable to optimize. 
We show with extensive experimental results that it can be applied to density estimation on synthetic data, image generation, image denoising, and training latent variable models with implicit encoders.

\end{abstract}

%% file: intro.tex
\section{Introduction}

Autoregressive models play a crucial role in modeling high-dimensional probability distributions. They have been successfully used to generate realistic images~\cite{oord2016pixel, salimans2017pixelcnn++}, high-quality speech~\cite{oord2016wavenet}, and complex decisions in games~\cite{vinyals2019grandmaster}. An autoregressive model defines a probability density as a product of conditionals using the chain rule. Although this factorization is fully general, 
autoregressive models typically rely on simple 
probability density functions for the conditionals (e.g. a Gaussian or a mixture of logistics)~\cite{salimans2017pixelcnn++} in the continuous case, which limits the expressiveness of the model.

To improve flexibility, energy-based models (EBM) represent a density in terms of an energy function, which does not need to be normalized. This enables more flexible neural network architectures, but requires new training strategies, since maximum likelihood estimation (MLE) is 
intractable due to the normalization constant (partition function).
Score matching (SM)~\cite{hyvarinen2005estimation} %
trains  
EBMs 
by minimizing the Fisher divergence (instead of KL divergence as in MLE)  between model and data distributions. It compares distributions in terms of their log-likelihood gradients (scores) and completely circumvents the intractable partition function. However, score matching requires computing the trace of the Hessian matrix of the model's log-density, which is expensive for high-dimensional data~\cite{martens2012estimating}.

To avoid calculating the partition function without losing scalability in high dimensional settings, 
we leverage the chain rule to decompose a high dimensional distribution matching problem into simpler univariate sub-problems.
Specifically, we propose a new divergence between distributions, named \mname{} (CSM), which depends only on the derivatives of \emph{univariate} log-conditionals (scores) of the model, instead of the full gradient as in score matching. %
CSM training is particularly efficient when the model is represented directly in terms of these univariate conditional scores. This is similar to a traditional autoregressive model, but with the advantage that conditional scores, unlike conditional distributions, do not need to be normalized. Similar to EBMs, removing the normalization constraint increases the flexibility of model families that can be used.

Leveraging existing and well-established autoregressive models, we design architectures where we can evaluate all dimensions in parallel for efficient training. %
During training, our CSM divergence can be optimized directly without the need of 
approximations~\cite{nash2019autoregressive, song2019sliced}, surrogate losses~\cite{kingma2013auto}, adversarial training~\cite{goodfellow2014generative} or extra sampling~\cite{du2019implicit}. 
We show with extensive experimental results that our method can be used for density estimation, data generation, image denoising and anomaly detection. We also illustrate that CSM can provide accurate score estimation required for variational inference with implicit distributions~\cite{huszar2017variational,song2019sliced} by providing better likelihoods and FID~\cite{heusel2017gans} scores compared to other training methods on image datasets.

%% file: background.tex
\section{Background}
Given i.i.d. samples $\{\bfx^{(1)},...,\bfx^{(N)}\} \subset \mathbb{R}^D$ from some unknown data
distribution $p(\bfx)$, we want to learn an unnormalized density $\Tilde{ q}_{\theta}(\bfx)$ as a parametric approximation to $p(\bfx)$. The unnormalized $\Tilde{ q}_{\theta}(\bfx)$ uniquely defines the following normalized probability density:
\begin{equation}
    q_{\theta}(\bfx)=\frac{\Tilde{q}_{\theta} (\bfx)}{Z(\theta)},\;  Z({\theta})=\int \Tilde{q}_{\theta}(\bfx)d\bfx,
    \label{eq:unnormalized_q}
\end{equation}
where $Z(\theta)$, the partition function, is generally intractable.

\subsection{Autoregressive Energy Machine}
To learn an unnormalized probabilistic model, \cite{nash2019autoregressive} proposes to approximate the normalizing constant using one dimensional importance sampling. Specifically, let $\bfx=(x_1,..., x_D)\in \mathbb{R}^D$. They first learn a set of one dimensional conditional energies $E_{\theta}(x_{d}| \bfx_{<d})\triangleq -\log \Tilde{q}_{\theta}(x_d|\bfx_{<d})$,
and then approximate the normalizing constants using importance sampling, which introduces an additional network to parameterize the proposal distribution. Once the partition function is approximated, they normalize the density to enable maximum likelihood training. However, approximating the partition function not only introduces bias into optimization but also requires extra computation and memory usage, lowering the training efficiency.

\subsection{Score Matching}
To avoid computing $Z(\theta)$, we can take the logarithm on both sides of \eqnref{eq:unnormalized_q} and obtain $\log q_{\theta}(\bfx)=\log \Tilde{q}_{\theta}(\bfx)-\log Z(\theta)$.  Since $Z(\theta)$ does not depend on $\bfx$, we can ignore the intractable
partition function $Z(\theta)$ when optimizing $\nabla_{\bfx}\log q_{\theta}(\bfx)$. In general, $\nabla_{\bfx}\log q_{\theta}(\bfx)$ and $\nabla_{\bfx}\log p(\bfx)$ are called the \textit{score} of $ q_{\theta}(\bfx)$ and $p(\bfx)$ respectively. Score matching (SM)~\cite{hyvarinen2005estimation} learns $q_{\theta}(\bfx)$ by matching the scores between $q_{\theta}(\bfx)$ and $p(\bfx)$ using the Fisher divergence:
\begin{align}
    L(q_{\theta};p) \triangleq \frac{1}{2}\mathbb{E}_{p} [\|\nabla_\bfx \log p(\bfx) - \nabla_\bfx \log q_\theta(\bfx)\|_2^2].
\label{eq:score_matching_original}
\end{align}
Ref.~\cite{hyvarinen2005estimation} shows that under certain regularity conditions $L(\theta;p)=J(\theta;p)+C$, where $C$ is a constant that does not depend on $\theta$ and $J(\theta; p)$ is defined as below:
\begin{align*}
    J(\theta;p) \triangleq \mathbb{E}_{p} \bigg[\frac{1}{2}\|\nabla_\bfx \log q_{\theta}(\bfx) \|_2^2+\text{tr} (\nabla^2_\bfx \log q_\theta(\bfx))\bigg],
\end{align*}
where $\text{tr}(\cdot)$ denotes the trace of a matrix. The above objective does not involve the intractable term $\nabla_{\bfx} \log p(\bfx)$. However, computing $\text{tr}(\nabla^2_\bfx \log q_\theta(\bfx))$ is in general expensive for high dimensional data. 
Given $\bfx \in \mathbb{R}^{D}$, a naive approach requires $D$ times more backward passes than computing the gradient $\nabla_\bfx \log q_{\theta}(\bfx)$~\cite{song2019sliced} in order to compute $\text{tr}(\nabla^2_\bfx \log q_\theta(\bfx))$, which is inefficient when $D$ is large. In fact, ref.~\cite{martens2012estimating} shows that within a constant number of
forward and backward passes, it is unlikely for an algorithm to be able to 
compute the diagonal of a Hessian matrix defined by any arbitrary computation graph.

%% file: method.tex
\section{\mname{}}

To make SM more scalable, we introduce \mname{} (CSM), a new divergence suitable for learning unnormalized statistical models. 
 We can factorize  any given data distribution $p(\bfx)$ and model distribution $q_{\theta}(\bfx)$ using the chain rule according to a common variable ordering:
\begin{equation*}
    p(\bfx) = \prod_{d=1}^{D} p(x_d|\bfx_{<d}),\; \quad q_{\theta}(\bfx) = \prod_{d=1}^{D} q_{\theta}(x_d|\bfx_{<d})
\end{equation*}
where $x_d\in \mathbb{R}$ stands for the $d$-th component of $\bfx$, and $\bfx_{<d}$ refers to all the entries with indices smaller than $d$ in $\bfx$.
Our key insight is that instead of directly matching the joint distributions, we can match the conditionals of the model $q_{\theta}(x_d|\bfx_{<d})$ to the conditionals of the data $p(x_d|\bfx_{<d})$ using the Fisher divergence. This decomposition results in simpler problems, which can be optimized efficiently using \emph{one-dimensional} score matching. For convenience, we denote the conditional scores of 
$q_{\theta}$ and $p$ as $s_{\theta, d}(\bfx_{<d}, x_d)\triangleq \frac{\partial}{\partial x_d} \log q_\theta(x_d | \bfx_{<d}): \mathbb{R}^{d-1}\times \mathbb{R} \rightarrow \mathbb{R}$
and
$s_{d}(\bfx_{<d}, x_d)\triangleq \frac{\partial}{\partial x_d} \log p(x_d | \bfx_{<d}):\mathbb{R}^{d-1}\times \mathbb{R} \rightarrow \mathbb{R}$ respectively. 
This gives us a new divergence termed \textit{\mname{} (CSM)}: 
\begin{align}
    L_{CSM}(q_{\theta}; p) = \frac{1}{2}\sum_{d=1}^D \mathbb{E}_{p(\bfx_{< d})} \mathbb{E}_{p(x_d|\bfx_{<d})} \bigg[ (s_d(\bfx_{<d}, x_d) - s_{\theta,d}(\bfx_{<d}, x_d)) ^2\bigg]. 
\label{eq:ar-sm-original}
\end{align}

This divergence is inspired by composite scoring rules~\cite{dawid2014theory}, 
a general technique to decompose  distribution-matching problems into lower-dimensional ones. As such, it bears some similarity with pseudo-likelihood, a composite scoring rule based on KL-divergence. As shown in the following theorem, it can be used as a learning objective to compare probability distributions:
\begin{restatable}[CSM Divergence]{theorem}{csmd}
\label{thm:csmd}
$L_{CSM}(q_{\theta}, p)$ vanishes if and only if $q_{\theta}(\bfx)=p(\bfx)$ a.e. 
\end{restatable}
\begin{proof}[Proof Sketch] If the distributions match, their derivatives (conditional scores) must be the same, hence $L_{CSM}$ is zero. If $L_{CSM}$ is zero, the conditional scores must be the same, and that uniquely determines the joints. See Appendix for a formal proof.
\end{proof}
\eqnref{eq:ar-sm-original} involves $s_{d}(\bfx)$, the unknown score function of the data distribution. Similar to score matching, we can apply integration by parts to obtain an equivalent but tractable expression: 
\begin{align}
J_{CSM}(\theta; p) = \sum_{d=1}^D \mathbb{E}_{p(\bfx_{< d})} \mathbb{E}_{p(x_d|\bfx_{<d})} \bigg[\frac{1}{2} s_{\theta,d}(\bfx_{<d}, x_d)^2 + \frac{\partial}{\partial x_d} s_{\theta,d}(\bfx_{<d},x_d)\bigg],
\label{eq:ar-sm}
\end{align}
The equivalence can be summarized using the following results:

\begin{restatable}[Informal]{theorem}{arobjective}
\label{thm:ar-objective}
 Under some regularity conditions, $L_{CSM}(\theta; p)=J_{CSM}(\theta; p)+C$ where $C$ is a constant that does not depend on $\theta$.
\end{restatable}

\begin{proof}[Proof Sketch] Integrate by parts the one-dimensional SM objectives. See Appendix for a proof.
\end{proof}
\begin{corollary}
Under some regularity conditions, $J_{CSM}(\theta, p)$ is minimized when $q_{\theta}(\bfx)=p(\bfx)$ a.e. 
\label{thm:valid_loss}
\end{corollary}

In practice, the expectation in $J_{CSM}(\theta; p)$ can be approximated by a sample average 
using the following unbiased estimator 
\begin{align}
\hat{J}_{CSM}(\theta; p)\triangleq \frac{1}{N}\sum_{i=1}^{N}\sum_{d=1}^{D}\bigg[\frac{1}{2}s_{\theta,d}(\bfx_{<d}^{(i)}, x_d^{(i)})^2 + \frac{\partial}{\partial x^{(i)}_d} s_{\theta,d}(\bfx_{<d}^{(i)}, x_d^{(i)}) \bigg],
\label{eq:csm:finitesample}
\end{align}
where $\{\bfx^{(1)},..., \bfx^{(N)}\}$ are i.i.d samples from $p(\bfx)$.
It is clear from \eqnref{eq:csm:finitesample} that evaluating $\hat{J}_{CSM}(\theta; p)$ is efficient as long as it is efficient to evaluate
$s_{\theta, d}(\bfx_{<d}, x_d)\triangleq \frac{\partial}{\partial x_d} \log q_\theta(x_d | \bfx_{<d})$ and its derivative $\frac{\partial}{\partial x_d} s_{\theta,d}(\bfx_{<d}, x_d)$. This in turn depends on how the model $q_\theta$ is represented. For example, if $q_\theta$ is an energy-based model defined in terms of an energy $\Tilde{q}_{\theta}$  as in \eqnref{eq:unnormalized_q}, computing  $q_\theta(x_d | \bfx_{<d})$ (and hence its derivative, $s_{\theta, d}(\bfx_{<d}, x_d)$) is generally intractable.  %
On the other hand, if $q_\theta$ is a traditional autoregressive model represented as a product of normalized conditionals, then $\hat{J}_{CSM}(\theta; p)$ will be efficient to optimize, but the normalization constraint may limit expressivity. In the following, we propose a parameterization tailored for CSM training, where we represent a joint distribution directly in terms of $s_{\theta, d}(\bfx_{<d},x_d), d=1, \cdots D$ without normalization constraints.

\section{Autoregressive conditional score models}

We introduce a new class of probabilistic models, named \emph{autoregressive conditional score models} (AR-CSM), defined as follows:
\begin{definition}
\label{def:arcsm}
An autoregressive conditional score model over $\mathbb{R}^D$ is a collection of $D$ functions
$\hat{s}_d(\bfx_{< d}, x_d):\mathbb{R}^{d-1} \times \mathbb{R} \to \mathbb{R}$, such that for all $d=1, \cdots, D$:
\begin{itemize}
     \item [1.] For all $\bfx_{< d} \in \mathbb{R}^{d-1}$,
     there exists a function
     $\mathcal{E}_d(\bfx_{< d}, x_d): \mathbb{R}^{d-1} \times \mathbb{R} \to \mathbb{R}$ such that $\frac{\partial}{\partial x_d}\mathcal{E}_d(\bfx_{< d}, x_d)$ exists, and
    $\frac{\partial}{\partial x_d}\mathcal{E}_d(\bfx_{< d}, x_d)=\hat{s}_d(\bfx_{< d}, x_d)$.  %
    \item [2.] For all $\bfx_{< d} \in \mathbb{R}^{d-1}$,  
    $Z_d(\bfx_{< d})\triangleq \int e^{\mathcal{E}_d(\bfx_{< d}, x_d)} d x_d$ exists and is finite (\ie, the improper integral \textit{w.r.t.} $x_d$ is convergent).
\end{itemize}
\end{definition}
Autoregressive conditional score models are an expressive family of probabilistic models for continuous data. In fact, there is a one-to-one mapping between the set of autoregressive conditional score models and a large set of probability densities over $\mathbb{R}^D$:
\begin{theorem}
\label{thm:bijection_family}
There is a one-to-one mapping between the set of autoregressive conditional score models over $\mathbb{R}^D$ and the set of probability density functions $q(\bfx)$ fully supported over $\mathbb{R}^D$ such that  $\frac{\partial}{\partial x_d}\log q(x_d|\bfx_{<d})$ exists for all $d$ and $\bfx_{< d} \in \mathbb{R}^{d-1}$. 
The mapping pairs conditional scores and densities such that
\[
\hat{s}_d(\bfx_{< d}, x_d) = \frac{\partial}{\partial x_d}\log q(x_d|\bfx_{<d})
\]
\end{theorem}
The key advantage of this representation is that the functions in Definition \ref{def:arcsm} are easy to parameterize (e.g., using neural networks) as the requirements 1 and 2 are typically easy to enforce. In contrast with typical autoregressive models, we do not require the functions in Definition \ref{def:arcsm} to be normalized. Importantly, Theorem~\ref{thm:bijection_family} does not hold for previous approaches that learn a single score function for the joint distribution~\cite{song2019sliced,song2019generative}, since the score model $s_\theta:\mathbb{R}^D \to \mathbb{R}^D$ in their case is not necessarily the gradient of any underlying joint density. 
In contrast, AR-CSM \emph{always} define a valid density through the mapping given by Theorem \ref{thm:bijection_family}.

In the following, we discuss how to use deep neural networks to parameterize autoregressive conditional score models (AR-CSM) defined in Definition \ref{def:arcsm}. To simplify notations, we hereafter use $\bfx$ to denote the arguments for $s_{\theta,d}$ and $s_{d}$ even when these functions depend on a subset of its dimensions.

\subsection{Neural AR-CSM models}
We propose to parameterize an AR-CSM based on existing autoregressive architectures for traditional (normalized) density models (\eg, PixelCNN++~\cite{salimans2017pixelcnn++}, MADE~\cite{germain2015made}). One important difference is that the output of standard autoregressive models at dimension $d$ depend only on $\bfx_{<d}$, yet we want the conditional score $s_{\theta, d}$ to also depend on $x_d$. 

To fill this gap, we use standard autoregressive models to parameterize a "context vector" $\bfc_d\in \mathbb{R}^{c}$ ($c$ is fixed among all dimensions) that depends only on $\bfx_{<d}$, and
then incorporate the dependency on $x_d$ by concatenating $\bfc_d$
and $x_d$ to get a $c+1$ dimensional vector $\bfh_d=[\textbf{c}_d, x_d]$. 
Next, we feed $\bfh_d$ into another neural network 
which outputs the scalar $s_{\theta, d} \in \mathbb{R}$ to model the conditional score. The network's parameters are shared across all dimensions similar to~\cite{nash2019autoregressive}. %
Finally, we can compute $\frac{\partial}{\partial x_d}s_{\theta, d}(\bfx)$ using automatic differentiation, and optimize the model directly with the CSM divergence.

Standard autoregressive models, such as PixelCNN++ and MADE, model the density with a prescribed probability density function (\eg, a Gaussian density) parameterized by functions of $\bfh_d$. In contrast, we remove the normalizing constraints of these density functions and therefore able to capture stronger correlations among dimensions with more expressive architectures.

\subsection{Inference and learning}
To sample from an AR-CSM model, we use one dimensional Langevin dynamics to sample from each dimension in turn. Crucially, Langevin dynamics only need the score function to sample from a density~\cite{parisi1981correlation,grenander1994representations}. In our case, scores are simply the univariate derivatives given by the AR-CSM.
Specifically, we use $s_{\theta,1}(x_1)$ to obtain a sample $\overline{x}_1 \sim q_{\theta}(x_1)$, then use $s_{\theta,2}(\overline{x}_1,x_2)$ to sample from $\overline{x}_2 \sim q_{\theta}(x_2 \mid \overline{x}_1)$ and so forth.
Compared to Langevin dynamics performed directly on a high dimensional space, one dimensional Langevin dynamics can converge faster under certain regularity conditions~\cite{roberts1996exponential}. See \appref{app:inference_langevin_dynamics} for more details.

During training, we use the CSM divergence (see ~\eqnref{eq:csm:finitesample}) 
to train the model. To deal with data distributions supported on low-dimensional manifolds and the difficulty of score estimation in low data density regions, we use noise annealing similar to \cite{song2019generative} with slight modifications: Instead of performing noise annealing as a whole, we perform noise annealing on each dimension individually. 
More details can be found in \appref{app:ar-csm}.

%% file: exp.tex
\section{Density estimation with AR-CSM}
In this section, we first compare the optimization performance of CSM with two other variants of score matching: Denoising Score Matching (DSM)~\cite{vincent2011connection} and Sliced Score Matching (SSM)~\cite{song2019sliced}, and compare the training efficiency of CSM with Score Matching (SM)~\cite{hyvarinen2005estimation}. Our results show that CSM is more stable to optimize and more scalable to high dimensional data compared to the previous score matching methods.
We then perform density estimation on 2-d synthetic datasets (see \appref{app:toy_dataset}) and three commonly used image datasets: MNIST, CIFAR-10~\cite{krizhevsky2009learning} and CelebA~\cite{liu2015deep}. We further show that our method can also be applied to image denoising and anomaly detection, illustrating broad applicability of our method.
\subsection{Comparison with other score matching methods}

\paragraph{Setup} To illustrate the scalability of CSM, we consider a simple setup of learning Gaussian distributions. 
We train an AR-CSM model with CSM and the other score matching methods on a fully connected network with 3 hidden layers.
We use comparable number of parameters for all the methods to ensure fair comparison.

\begin{figure}[htbp]
    \centering
    \begin{subfigure}[b]{0.3\textwidth}
        \includegraphics[width=\textwidth]{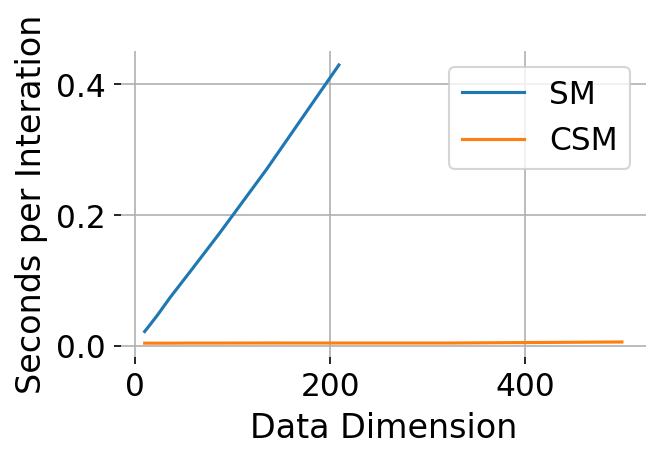}
        \caption{Training time per iteration.}
        \label{fig:training-efficiency}
    \end{subfigure}
    ~
    \begin{subfigure}[b]{0.3\textwidth}
        \includegraphics[width=\textwidth]{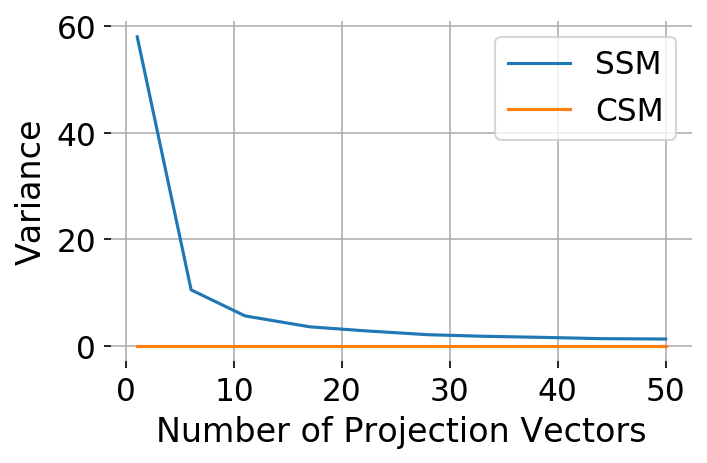}
        \caption{Variance comparison.}
        \label{fig:ssm_var}
    \end{subfigure}
    \caption{Comparison with SSM and SM in terms of loss variance and computational efficiency.}
\end{figure}

\paragraph{CSM vs. SM} 
In \figref{fig:training-efficiency}, we show the time per iteration of CSM versus the original score matching (SM) method~\cite{hyvarinen2005estimation} on multivariate Gaussians with different data dimensionality.
We find that the training speed of SM degrades linearly as a function of the data dimensionality. Moreover, the memory required grows rapidly \textit{w.r.t} the data dimension, which triggers memory error on 12 GB TITAN Xp GPU when the data dimension is approximately $200$. On the other hand, for CSM, the time required stays stable as the data dimension increases due to parallelism, and no memory errors occurred throughout the experiments. As expected, traditional score matching (SM) does not scale as well as CSM for high dimensional data. Similar results on SM were also reported in \cite{song2019sliced}.

\paragraph{CSM vs. SSM}
We compare CSM with Sliced Score Matching (SSM)~\cite{song2019sliced}, a recently proposed score matching variant, on learning a representative Gaussian $\mathcal{N}(0, 0.1^2 I)$ of dimension $100$ in  \figref{fig:sm_comparision_loss} (2 rightmost panels). While CSM converges rapidly, SSM does not converge even after 20k iterations due to the large variance of random projections. We compare the variance of the two objectives in \figref{fig:ssm_var}. In such a high-dimensional setting, SSM would require a large number of projection vectors for variance reduction, which requires extra computation and could be prohibitively expensive in practice. By contrast, CSM is a deterministic objective function that is more stable to optimize.
This again suggests that CSM might be more suitable to be used in high-dimensional data settings compared to SSM.

\paragraph{CSM vs. DSM}
Denoising score matching (DSM)~\cite{vincent2011connection} is perhaps the most scalable score matching alternative available, and has been applied to high dimensional score matching problems~\cite{song2019generative}. However, DSM estimates the score of the data distribution after it has been convolved with Gaussian noise with variance $\sigma^2I$. In \figref{fig:sm_comparision_loss}, we use various noise levels $\sigma$ for DSM, and compare the performance of CSM with that of DSM. We observe that although DSM shows reasonable performance when $\sigma$ is sufficiently large, the training can fail to converge for small $\sigma$. In other words, for DSM, there exists a tradeoff between optimization performance and the bias introduced due to noise perturbation for the data. CSM on the other hand does not suffer from this problem, and converges faster than DSM.

\paragraph{Likelihood comparison}
To better compare density estimation performance of DSM, SSM and CSM, we train a MADE~\cite{germain2015made} model with tractable likelihoods on MNIST, a more challenging data distribution, using the three variants of score matching objectives. We report the negative log-likelihoods in \figref{fig:mle_loss}. The loss curves in \figref{fig:mle_loss} align well with our previous discussion. For DSM, a smaller $\sigma$ introduces less bias, but also makes training slower to converge. For SSM, training convergence can be handicapped by the large variance due to random projections. In contrast, CSM can converge quickly without these difficulties. This clearly demonstrates the efficacy of CSM over the other score matching methods for density estimation.

\begin{figure}[t]
    \centering
    \begin{subfigure}[b]{0.19\textwidth}
        \includegraphics[width=\textwidth]{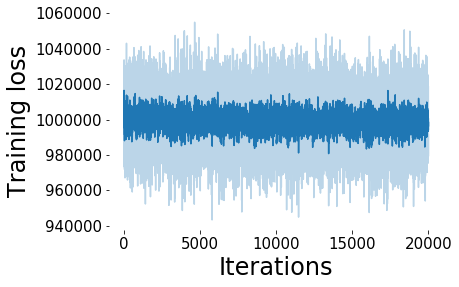}
        \caption*{DSM ($\sigma=0.01$)}
        \label{fig:DSM_variance_loss_0.01}
    \end{subfigure}
    \begin{subfigure}[b]{0.19\textwidth}
        \includegraphics[width=\textwidth]{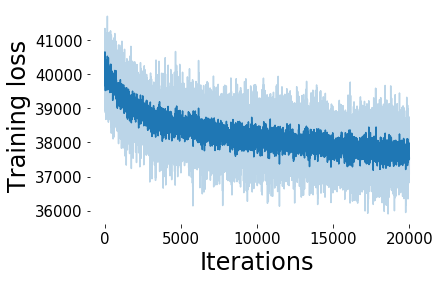}
        \caption*{DSM ($\sigma=0.05$)}
        \label{fig:DSM_variance_loss_0.05}
    \end{subfigure}
     \begin{subfigure}[b]{0.19\textwidth}
        \includegraphics[width=\textwidth]{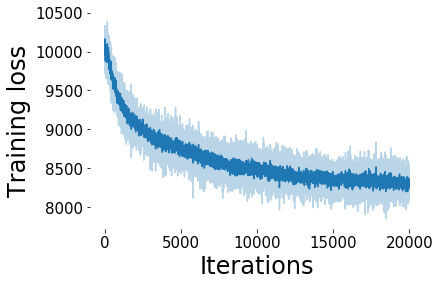}
        \caption*{DSM ($\sigma=0.1$)}
        \label{fig:DSM_variance_loss_0.1}
    \end{subfigure}
    \begin{subfigure}[b]{0.19\textwidth}
        \includegraphics[width=\textwidth]{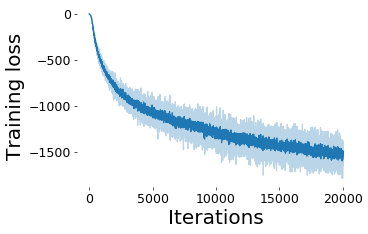}
        \caption*{SSM ($\sigma=0$)}
        \label{fig:SSM_variance_loss}
    \end{subfigure}
    \begin{subfigure}[b]{0.19\textwidth}
        \includegraphics[width=\textwidth]{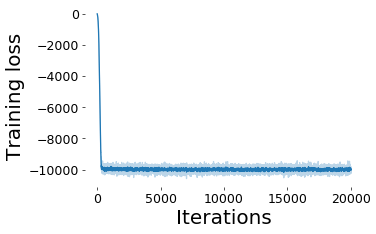}
        \caption*{CSM ($\sigma=0$)}
        \label{fig:AR-SM_variance_loss}
    \end{subfigure}
    \caption{Training losses for DSM, SSM and CSM on $100$-d Gaussian distribution $\mathcal{N}(0, 0.1^2I)$. Note the vertical axes are different across methods as they optimize different losses.}
    \label{fig:sm_comparision_loss}
\end{figure}

\begin{figure}[t]
    \centering
    \begin{subfigure}[b]{0.23\textwidth}
    \label{fig:mle_loss}
        \adjincludegraphics[width=\textwidth, trim={{.75\width} {.0\width} {.0\width} {.1\height}}, clip=true, ]{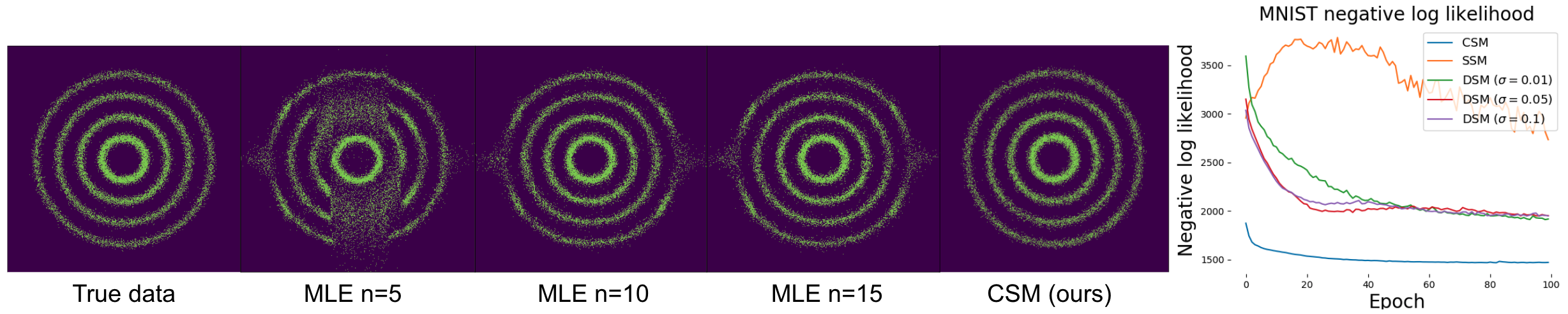}
        \caption{MNIST negative log-likelihoods}
        \label{fig:mle_loss}
    \end{subfigure}
    ~
    \begin{subfigure}[b]{0.74\textwidth}
        \adjincludegraphics[width=\textwidth, trim={{.0\width} {.0\width} {.25\width} {.1\height}}, clip=true, ]{img/toy-csm.png}
        \caption{2-d synthetic dataset samples from MADE MLE baselines with $n$ mixture of logistics and an AR-CSM model trained by CSM.}
        \label{fig:2d_sm_samples}
    \end{subfigure}
    \caption{Negative log-likelihoods on MNIST and samples on a 2-d synthetic dataset.}
\end{figure}

\subsection{Learning 2-d synthetic data distributions with AR-CSM}
In this section, we focus on a 2-d synthetic data distribution (see \figref{fig:2d_sm_samples}). We compare the sample quality of an autoregressive model trained by maximum likelihood estimation (MLE) and an AR-CSM model trained by CSM. We use a MADE model with $n$ mixture of logistic components for the MLE baseline experiments. We also use a MADE model as the autoregressive architecture for the AR-CSM model. 
To show the effectiveness of our approach, we use strictly fewer parameters for the AR-CSM model than the baseline MLE model. Even with fewer parameters, the AR-CSM model trained with CSM is still able to generate better samples than the MLE baseline (see \figref{fig:2d_sm_samples}).

\subsection{Learning high dimensional distributions over images with AR-CSM}
In this section, we show that our method is also capable of modeling natural images. We focus on three image datasets, namely MNIST, CIFAR-10, and CelebA.

\textbf{Setup}\; 
We select two existing autoregressive models --- MADE~\cite{germain2015made} and PixelCNN++~\cite{salimans2017pixelcnn++}, as the autoregressive architectures for AR-CSM. For all the experiments, we use a shallow fully connected network to transform the context vectors to the conditional scores for AR-CSM. 
Additional details can be found in \appref{app:ar-csm}. 

\begin{figure}[!ht]
    \centering
    \begin{subfigure}[b]{0.21\linewidth}
          \adjincludegraphics[width=\textwidth, trim={{.2\width} {.2\width} {.2\width} {.3\width}}, clip=true, ]{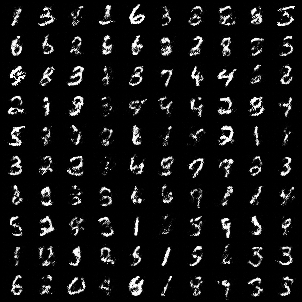}
    \end{subfigure}
    \begin{subfigure}[b]{0.21\linewidth}
         \adjincludegraphics[width=\textwidth, trim={{.2\width} {.2\width} {.2\width} {.3\width}}, clip=true, ]{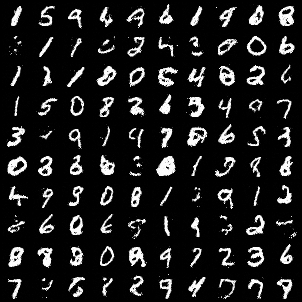}
    \end{subfigure}
    \begin{subfigure}[b]{0.28\linewidth}
         \adjincludegraphics[width=\textwidth, trim={{.1\width} {.2\width} {.1\width} {.3\width}}, clip=true, ]{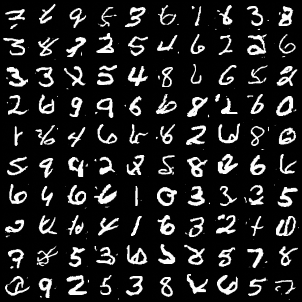}
    \end{subfigure}
    \begin{subfigure}[b]{0.28\linewidth}
          \adjincludegraphics[width=\textwidth, trim={{.1\width} {.2\width} {.1\width} {.3\width}}, clip=true, ]{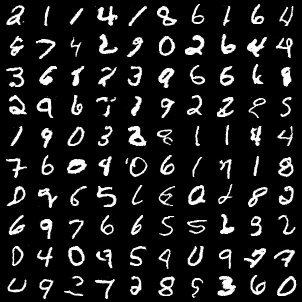}
    \end{subfigure}
    \hfill
    \begin{subfigure}[b]{0.21\linewidth}
          \adjincludegraphics[width=\textwidth, trim={{.2\width} {.2\width} {.2\width} {.3\width}}, clip=true, ]{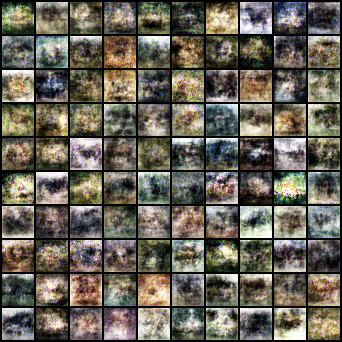}
    \caption*{MADE MLE}
    \end{subfigure}
    \begin{subfigure}[b]{0.21\linewidth}
         \adjincludegraphics[width=\textwidth, trim={{.2\width} {.2\width} {.2\width} {.3\width}}, clip=true, ]{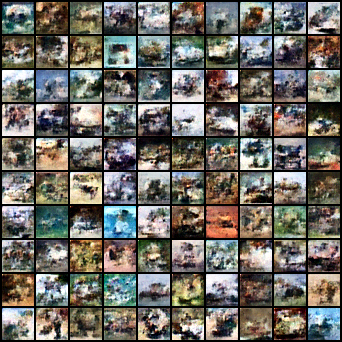}
    \caption*{MADE CSM}
    \end{subfigure}
    \begin{subfigure}[b]{0.28\linewidth}
         \adjincludegraphics[width=\textwidth, trim={{.1\width} {.2\width} {.1\width} {.3\width}}, clip=true, ]{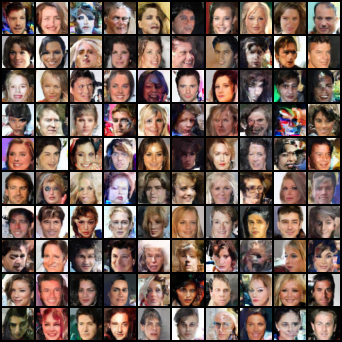}
    \caption*{PixelCNN++ MLE}
    \end{subfigure}
    \begin{subfigure}[b]{0.28\linewidth}
          \adjincludegraphics[width=\textwidth, trim={{.1\width} {.2\width} {.1\width} {.3\width}}, clip=true, ]{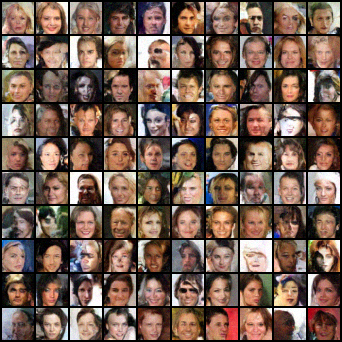}
    \caption*{PixelCNN++ CSM}
    \end{subfigure}
   \caption{Samples from MADE and PixelCNN++ using MLE and CSM.}
\label{fig:direct_samples}
\end{figure}

\textbf{Results}\; 
We compare the samples from AR-CSM with the ones from MADE and PixelCNN++ with similar autoregressive architectures but trained via maximum likelihood estimation. Our AR-CSM models have comparable number of parameters as the maximum-likelihood counterparts. We observe that the MADE model trained by CSM is able to generate sharper and higher quality samples than its maximum-likelihood counterpart using Gaussian densities (see \figref{fig:direct_samples}).
For PixelCNN++, we observe more digit-like samples on MNIST, and less shifted colors on CIFAR-10 and CelebA than its maximum-likelihood counterpart using mixtures of logistics (see \figref{fig:direct_samples}).
We provide more samples in \appref{app:ar-csm}.

\subsection{Image denoising with AR-CSM}
\begin{figure}[!ht]
    \centering
    \begin{subfigure}[b]{0.24\linewidth}
         \adjincludegraphics[width=\textwidth, trim={{.1\width} {.3\width} {.1\width} {.3\width}}, clip=true, ]{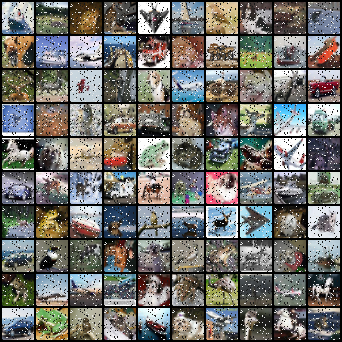}
    \caption*{S\&P Noise}
    \end{subfigure}
    \begin{subfigure}[b]{0.24\linewidth}
         \adjincludegraphics[width=\textwidth, trim={{.1\width} {.3\width} {.1\width} {.3\width}}, clip=true, ]{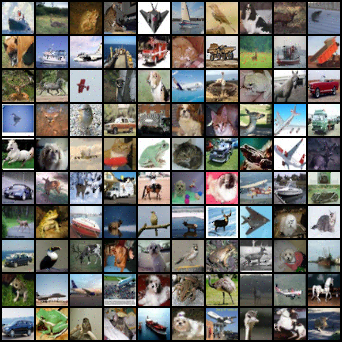}
    \caption*{Denoised}
    \end{subfigure}
    \begin{subfigure}[b]{0.24\linewidth}
         \adjincludegraphics[width=\textwidth, trim={{.1\width} {.3\width} {.1\width} {.3\width}}, clip=true, ]{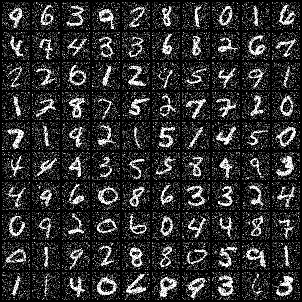}
    \caption*{Gaussian Noise}
    \end{subfigure}
    \begin{subfigure}[b]{0.24\linewidth}
         \adjincludegraphics[width=\textwidth, trim={{.1\width} {.3\width} {.1\width} {.3\width}}, clip=true, ]{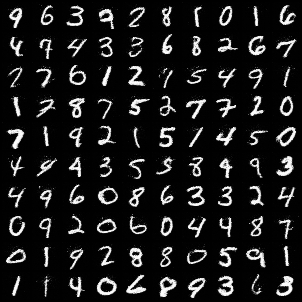}
    \caption*{Denoised}
    \end{subfigure}
   \caption{Salt and pepper denoising on CIFAR-10. Autoregressive single-step denoising on MNIST. %
   }
    \label{fig:denoise}
\end{figure}

Besides image generation, AR-CSM can also be used for image denoising. In \figref{fig:denoise}, we apply $10\%$ "Salt and Pepper" noise to the images in CIFAR-10 test set and apply Langevin dynamics sampling to restore the images. We also show that AR-CSM can be used for single-step denoising~\cite{saremi2018deep,vincent2011connection} and report the denoising results for MNIST, with noise level $\sigma=0.6$ in the rescaled space in \figref{fig:denoise}. These results qualitatively demonstrate the effectiveness of AR-CSM for image denoising, showing that our models are sufficiently expressive to capture complex distributions and solve difficult tasks.

\subsection{Out-of-distribution detection with AR-CSM}
\begin{wrapfigure}[9]{O}{7.0cm}
\vspace{-1em}
\begin{center}
\begin{adjustbox}{max width=1.0\linewidth}
    \begin{tabular}{p{2.2cm} c c c c}
        \toprule
        Model & PixelCNN++ & GLOW & EBM &AR-CSM(Ours)\\
        \midrule
        SVHN &0.32 &0.24 & 0.63 &\textbf{0.68}\\
        Const Uniform &0.0 &0.0 &0.30 &\textbf{0.57}\\
        Uniform &\textbf{1.0} &\textbf{1.0} & \textbf{1.0} &0.95\\
        Average &0.44  &0.41  &0.64  &\textbf{0.73}\\
        \bottomrule
    \end{tabular} 
\end{adjustbox}
\end{center}
\captionof{table}{AUROC scores for models trained on CIFAR-10.}
\label{tab:auroc_table}
\end{wrapfigure}

We show that the AR-CSM model can also be used for out-of-distribution (OOD) detection. In this task, the generative model is required to produce a statistic (\eg, likelihood, energy) such that the outputs of in-distribution examples can be distinguished from those of the out-of-distribution examples. We find that $h_{\theta}(\bfx)\triangleq  \sum_{d=1}^{D}s_{\theta, d}(\bfx)$ is an effective statistic for OOD. In \tabref{tab:auroc_table}, we compare the Area Under the Receiver-Operating Curve (AUROC) scores obtained by AR-CSM using $h_{\theta}(\bfx)$ with the ones obtained by PixelCNN++~\cite{salimans2017pixelcnn++}, Glow~\cite{kingma2018glow} and EBM~\cite{du2019implicit} using relative log likelihoods. We use SVHN, constant uniform and uniform as OOD distributions following ~\cite{du2019implicit}.
We observe that our method can perform comparably or better than existing generative models.

%% file: exp_vae.tex
\section {VAE training with implicit encoders and CSM}
In this section, we show that CSM can also be used to improve variational inference with implicit distributions~\cite{huszar2017variational}.
Given a latent variable model $p_{\theta}(\bfx, \bfz)$, where $\bfx$ is the observed variable and $\bfz\in \mathbb{R}^{D}$ is the latent variable, a Variational Auto-Encoder (VAE)~\cite{kingma2013auto} contains an encoder $q_{\phi}(\bfz | \bfx)$ and a decoder $p_{\theta}(\bfx|\bfz)$ that are jointly trained by maximizing the evidence lower bound (ELBO)
\begin{align}
    \mathbb{E}_{p_{data}(\bfx)}[\mathbb{E}_{q_{\phi}(\bfz|\bfx)}\log p_{\theta}(\bfx|\bfz)p(\bfz)-\mathbb{E}_{q_{\phi}(\bfz|\bfx)}\log q_{\phi}(\bfz|\bfx)],
\label{eq:elbo}
\end{align} 
Typically,  $q_{\phi}(\bfz|\bfx)$ is chosen to be a simple \emph{explicit} distribution such that the entropy term in Equation~(\ref{eq:elbo}),
$H(q_{\phi}(\cdot|\bfx))\triangleq-\mathbb{E}_{q_{\phi}(\bfz|\bfx)}[\log q_{\phi}(\bfz|\bfx)]$,
is tractable. To increase model flexibility, we can parameterize the encoder using implicit distributions---distributions that can be sampled tractably but do not have tractable densities (\eg, the generator of a GAN~\cite{goodfellow2014generative}). 
The challenge is that evaluating $H(q_{\phi}(\cdot|\bfx)) $ and its gradient  $\nabla_{\phi}H(q_{\phi}(\cdot|\bfx)))$ becomes intractable.

Suppose $z_d\sim q_{\phi}(z_d| \bfz_{<d}, \bfx)$ can be reparameterized as $g_{\phi, d}(\bfe_{\le d},\bfx)$, where $g_{\phi, d}$ is a deterministic mapping and $\bfe$ is a $D$ dimensional random variable. 
We can write the gradient of the entropy with respect to $\phi$ as
\begin{align*}
    \nabla_{\phi}H(q_\phi(\cdot|\bfx)))
    &=-\sum_{d=1}^{D}\mathbb{E}_{p(\bfe_{<d})}\mathbb{E}_{p(\bfe_d)}\bigg[\frac{\partial}{\partial z_d}\log q_{\phi}(z_d|\bfz_{<d}, \bfx)|_{z_d=g_{\phi, d}(\bfe_{\le d}, \bfx)}\nabla_{\phi}g_{\phi, d}(\bfe_{\le d}, \bfx)\bigg],
\end{align*}
where $\nabla_{\phi}g_{\phi, d}(\bfe_{\le d}, \bfx)$ is usually easy to compute and $\frac{\partial}{\partial z_d}\log q_{\phi}(z_d|\bfz_{<d}, \bfx)$ can be approximated by score estimation using CSM. We provide more details in \appref{app:vae_extra}.

\paragraph{Setup}
We train VAEs using the proposed method on two image datasets -- MNIST and CelebA. We follow the setup in \cite{song2019sliced} (see \appref{app:vae_setup}) and compare our method with ELBO, and three other methods, namely SSM~\cite{song2019sliced}, Stein~\cite{stein1981estimation}, and Spectral~\cite{shi2018spectral}, that can be used to train implicit encoders \cite{song2019sliced}. Since SSM can also be used to train an AR-CSM model, we denote the AR-CSM model trained with SSM as SSM-AR.
Following the settings in \cite{song2019sliced}, we report the likelihoods estimated by AIS~\cite{neal2001annealed} for MNIST, and FID scores~\cite{heusel2017gans} for CelebA.
We use the same decoder for all the methods, and encoders sharing similar architectures with slight yet necessary modifications. 
We provide more details in \appref{app:vae_extra}. 

\paragraph{Results}
We provide negative log-likelihoods (estimated by AIS) on MNIST and the FID scores on CelebA in \tabref{tab:vae_table}. We observe that CSM is able to marginally outperform other methods in terms of the metrics we considered. 
We provide VAE samples for our method in \figref{fig:csm_vae_samples}. Samples for the other methods can be found in \appref{app:vae_samples}.

\begin{table}
\begin{minipage}[t]{0.43\linewidth}
\vspace{0pt}
\begin{center}
 \begin{adjustbox}{max width=0.9\linewidth}
    \begin{tabular}{c|cc||c}
    	\Xhline{3\arrayrulewidth} \bigstrut
    	& \multicolumn{2}{c||}{MNIST (AIS)} & \multicolumn{1}{c}{CelebA (FID)}\\
    	\Xhline{1\arrayrulewidth}\bigstrut
    	Latent Dim &  8 & 16 & 32\\
    	\Xhline{1\arrayrulewidth}\bigstrut
        ELBO &96.74 & 91.82  &66.31\\
        Stein &96.90 &88.86 &108.84\\
        Spectral &96.85 &88.76 &121.51\\
        SSM &95.61 &88.44 &62.50\\
        SSM-AR &95.85 &88.98 &66.88\\
        CSM (Ours) &\textbf{95.02}  &\textbf{88.42} &\textbf{62.20}\\
        \Xhline{3\arrayrulewidth}
\end{tabular}
\end{adjustbox}
\end{center}
\captionof{table}{VAE results on MNIST and CelebA.}
\label{tab:vae_table}
\end{minipage}\hfill
\begin{minipage}[t]{0.55\linewidth}
\begin{center}
    \centering
    \vspace{0pt}
    \begin{subfigure}{0.48\linewidth}
        \adjincludegraphics[width=\textwidth,trim={{.2\width} {.4\width} 0 0}, clip=true]{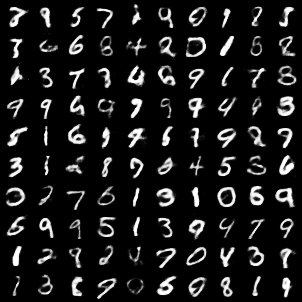} %
    \end{subfigure}
    \begin{subfigure}{0.48\linewidth}
        \adjincludegraphics[width=\textwidth,trim={{.2\width} {.4\width} 0 0}, clip=true]{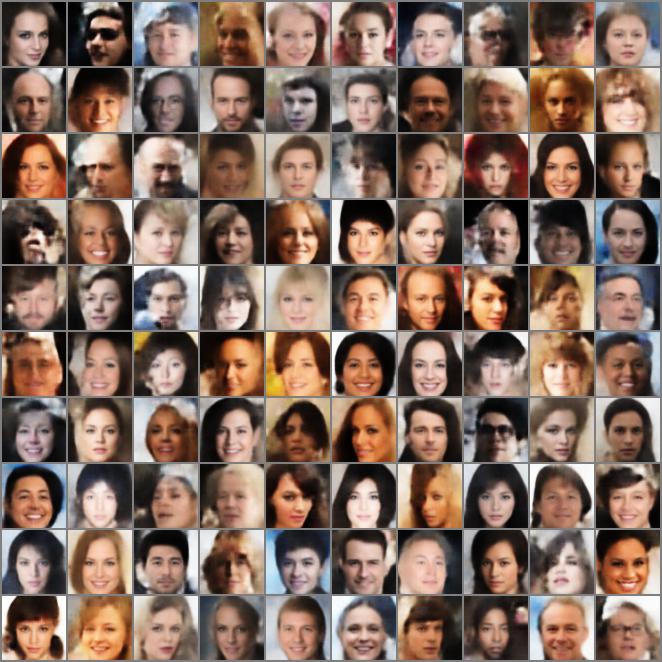}
    \end{subfigure}
\end{center}
\captionof{figure}{CSM VAE MNIST and CelebA samples.
}
\label{fig:csm_vae_samples}
\end{minipage}
\end{table}

%% file: related.tex
\section{Related work}
\label{sec:related_work}
 Likelihood-based deep generative models (\eg, flow models, autoregressive models) have been widely used for modeling high dimensional data distributions. Although such models have achieved promising results, they tend to have extra constraints which could limit the model performance. For instance, flow ~\cite{dinh2016density,kingma2018glow} and autoregressive~\cite{van2016conditional,oord2016wavenet} models require  normalized densities, while variational auto-encoders (VAE)~\cite{kingma2013auto} need to use surrogate losses. 

 Unnormalized statistical models allow one to use more flexible networks, but require new training strategies.
 Several approaches have been proposed to train unnormalized statistical models, all with certain types of limitations. 
 Ref.~\cite{du2019implicit} proposes to use Langevin dynamics together with a sample replay buffer to train an energy based model, which requires more iterations over a deep neural network for sampling during training. 
Ref.~\cite{yu2020training} proposes a variational framework to train energy-based models by minimizing general $f$-divergences, which also requires expensive Langevin dynamics to obtain samples during training.
Ref.~\cite{nash2019autoregressive} approximates the unnormalized density using importance sampling, which introduces bias during optimization and requires extra computation during training. 
There are other approaches that focus on modeling the log-likelihood gradients (scores) of the distributions. For instance, 
score matching (SM)~\cite{hyvarinen2005estimation} trains an unnormalized model by minimizing Fisher divergence, which introduces a new term that is expensive to compute for high dimensional data. Denoising score matching~\cite{vincent2011connection} is a variant of score matching that is fast to train. However, the performance of denoising score matching can be very sensitive to the perturbed noise distribution and heuristics have to be used to select the noise level in practice. Sliced score matching~\cite{song2019sliced} approximates SM by projecting the scores onto random vectors. Although it can be used to train high dimensional data much more efficiently than SM, it provides a trade-off between computational complexity and variance introduced while approximating the SM objective. By contrast, CSM is a deterministic objective function that is efficient and stable to optimize.

%% file: conclusion.tex
\section{Conclusion}
We propose a divergence between distributions, named \mname{} (CSM), which depends only on the derivatives of univariate log-conditionals (scores) of the model. Based on CSM divergence, we introduce a family of models dubbed AR-CSM, which allows us to 
expand the capacity of existing autoregressive likelihood-based models by removing the normalizing constraints of conditional distributions.
Our experimental results demonstrate good performance on density estimation, data generation, image denoising, anomaly detection and training VAEs with implicit encoders. Despite the empirical success of AR-CSM, sampling from the model is relatively slow since each variable has to be sampled sequentially according to some order. It would be interesting to investigate methods that accelerate the sampling procedure in AR-CSMs, or consider more efficient variable orders that could be learned from data.

\section*{Broader Impact}
The main contribution of this paper is theoretical---a new divergence between distributions and a related class of generative models. We do not expect any direct impact on society. The models we trained using our approach and used in the experiments have been learned using classic dataset and have capabilities substantially similar to existing models (GANs, autoregressive models, flow models): generating images, anomaly detection, denoising. As with other technologies, these capabilities can have both positive and negative impact, depending on their use. For example, anomaly detection can be used to increase safety, but also possibly for surveillance. Similarly, generating images can be used to enable new art but also in malicious ways.

\section*{Acknowledgments and Disclosure of Funding}
This research was supported by TRI, Amazon AWS, NSF (\#1651565, \#1522054, \#1733686), ONR (N00014-19-1-2145), AFOSR (FA9550-19-1-0024), and FLI.

%% file: app.tex
\newpage
\appendix
\section{Proofs}
\subsection{Regularity conditions}
The following regularity conditions are needed for identifiability and integration by parts.

We assume that for every $\bfx_{<d}$ and for any $\theta$
 \begin{enumerate}
     \item   $\frac{\partial}{\partial x_d}\log p(x_d|\bfx_{<d})$ and  $\frac{\partial}{\partial x_d} \log q_\theta(x_d | \bfx_{<d})$ are continuously differentiable over $\mathbb{R}$.
     \item $\mathbb{E}_{x_d \sim  p(x_d|\bfx_{<d})}[\left(\frac{\partial \log p(x_d|\bfx_{<d})}{\partial x_d}\right)^2]$ and $\mathbb{E}_{ x_d \sim p(x_d|\bfx_{<d})}[\left(\frac{\partial \log q_\theta(x_d|\bfx_{<d})}{\partial x_d}\right)^2]$ are finite.
     \item  $\lim_{|x_d| \rightarrow \infty} p(x_d|\bfx_{<d}) \frac{\partial \log q_\theta(x_d|\bfx_{<d})}{\partial x_d} = 0$.
 \end{enumerate}
 
\subsection[Proofs]{Proof of Theorem~\ref{thm:csmd} (See page~\pageref{thm:csmd})}
\csmd*
\begin{proof}
It is known that the Fisher divergence
\begin{align}
    L(q_{\theta};p) =\frac{1}{2} \mathbb{E}_{p} \bigg[\|\nabla_\bfx \log p(\bfx) - \nabla_\bfx \log q_\theta(\bfx)\|_2^2\bigg]
\label{eq:score_matching_original}
\end{align}
is a strictly proper scoring rule, and $L(q_{\theta};p) \geq 0$ and vanishes if and only if $q_{\theta} = p$ almost everywhere~\cite{hyvarinen2005estimation}.

Recall
\begin{align}
    L_{CSM}(q_{\theta}; p) = \frac{1}{2}\sum_{d=1}^D \mathbb{E}_{p(\bfx_{< d})} \mathbb{E}_{p(x_d|\bfx_{<d})} \bigg[ (s_d(\bfx) - s_{\theta,d}(\bfx)) ^2\bigg]. 
\end{align}
which we can rewrite as
\begin{align}
    L_{CSM}(q_{\theta}; p) = \sum_{d=1}^D \mathbb{E}_{p(\bfx_{< d})} \bigg[ 
    L\left(q_{\theta}(x_d \mid \bfx_{< d});p(x_d \mid \bfx_{< d})\right)
    \bigg]. 
\end{align}

When $q_{\theta} = p$ almost everywhere, we have
\begin{align}
    q_{\theta}(\bfx_{\le d}) &= q_{\theta}(\bfx_{<d}) q_\theta (x_d \mid \bfx_{<d}) = p(\bfx_{\le d}) = p(\bfx_{<d}) p(x_d \mid \bfx_{<d}) = q_\theta(\bfx_{<d}) p (x_d \mid \bfx_{<d}) \quad a.e. \label{eqn:1}
\end{align}
Let $\mcal{A} \triangleq \{ \bfx_{<d} \mid q_{\theta}(\bfx_{<d}) > 0\}$. We first observe that when $\bfx_{<d} \in \mcal{A}$, Eq.~\eqref{eqn:1} implies that $q_\theta (x_d \mid \bfx_{<d}) = p(x_d \mid \bfx_{<d})~~ a.e$, and subsequently, $(s_d(\bfx) - s_{\theta,d}(\bfx))^2 = 0 ~~a.e$. Therefore
\begin{align*}
    L_{CSM}(q_{\theta}; p) &= \frac{1}{2}\sum_{d=1}^D \mathbb{E}_{p(\bfx_{< d})} \mathbb{E}_{p(x_d|\bfx_{<d})} \bigg[ (s_d(\bfx) - s_{\theta,d}(\bfx)) ^2\bigg]\\
    &= \frac{1}{2}\sum_{d=1}^D \mathbb{E}_{p(\bfx_{<d})}\bigg[\mbb{I}[\bfx_{<d} \in \mcal{A}]\mathbb{E}_{p(x_d|\bfx_{<d})} \bigg[ (s_d(\bfx) - s_{\theta,d}(\bfx)) ^2\bigg] \bigg]\\
    &= 0
\end{align*}

Now assume $L_{CSM}(q_{\theta}; p)=0$. Because $L \geq 0$, $\mathbb{E}_{p(\bfx_{< d})} \bigg[ 
    L\left(q_{\theta}(x_d \mid \bfx_{< d});p(x_d \mid \bfx_{< d})\right)
    \bigg] \geq 0$ which means every term in the sum must be zero
\begin{align*}
 \mathbb{E}_{p(\bfx_{< d})} \bigg[ 
    L\left(q_{\theta}(x_d \mid \bfx_{< d});p(x_d \mid \bfx_{< d})\right)
    \bigg] = 0 \ \ \forall d
\end{align*}
and $L\left(q_{\theta}(x_d \mid \bfx_{< d});p(x_d \mid \bfx_{< d})\right) = 0$ $p(\bfx_{< d})$-almost everywhere.  
Let's show that $q_{\theta}(\bfx_{\le d})=p(\bfx_{\le d})$ almost everywhere using induction. When $d=1$, $L\left(q_{\theta}(x_1);p(x_1)\right) = 0$ almost everywhere implies $q_{\theta}(x_1)=p(x_1)$ almost everywhere. Assume the hypothesis holds when $d=k$, that is $q_{\theta}(\bfx_{\le k})=p(\bfx_{\le k}))$ almost everywhere. Using the fact that $L\left(q_{\theta}(x_{k+1} \mid \bfx_{< k+1});p(x_{k+1} \mid \bfx_{< k+1})\right) = 0$ $p(\bfx_{< k+1})$-almost everywhere (\ie, $p(\bfx_{\le k})$-almost everywhere), we have $$q_{\theta}(x_{\le k+1})=q_{\theta}(x_{k+1}\mid \bfx_{<k+1})q_{\theta}(\bfx_{\le k})=p(x_{k+1}\mid \bfx_{<k+1})p(\bfx_{\le k})=p(\bfx_{\le k+1}) \quad a.e.$$ 
Thus, the hypothesis holds when $d=k+1$. By induction hypothesis, we have $q_{\theta}(\bfx_{\le d})=p(\bfx_{\le d}) \quad a.e.$ for any $d$. In particular, $$q_{\theta}(\bfx)=q_{\theta}(\bfx_{\le D})=p(\bfx_{\le D})=p(\bfx) \quad a.e.$$
\end{proof}

\subsection[Proofs]{Proof of Theorem~\ref{thm:ar-objective} (See page~\pageref{thm:ar-objective})}
\begin{customthm}{\ref{thm:ar-objective}}[Formal Statement]
 $L_{CSM}(\theta; p)=J_{CSM}(\theta; p)+C$ where $C$ is a constant does not depend on $\theta$.
 \end{customthm}
\begin{proof}
Recall
\begin{align}
    L_{CSM}(q_{\theta}; p) = \frac{1}{2}\sum_{d=1}^D \mathbb{E}_{p(\bfx_{< d})} \mathbb{E}_{p(x_d|\bfx_{<d})} \bigg[ (s_d(\bfx) - s_{\theta,d}(\bfx)) ^2\bigg]. 
\end{align}
which we can rewrite as
\begin{align}
    L_{CSM}(q_{\theta}; p) = \sum_{d=1}^D \mathbb{E}_{p(\bfx_{< d})} \bigg[ 
    L\left(q_{\theta}(x_d \mid \bfx_{< d});p(x_d \mid \bfx_{< d})\right)
    \bigg]. 
    \label{eq:csmsimp}
\end{align}
where $L(q_{\theta};p)$ is the Fisher divergence. 

Under the assumptions, we can use Theorem 1 from \cite{hyvarinen2005estimation} which shows that $L(q_\theta;p)=J(\theta;p)+C$, where $C$ is a constant independent of $\theta$ and $J(q_\theta; p)$ is defined as below:
\begin{align*}
    J(q_\theta;p) = \mathbb{E}_{p} \bigg[\frac{1}{2}\|\nabla_\bfx \log q_{\theta}(\bfx) \|_2^2+\text{tr} (\nabla^2_\bfx \log q_\theta(\bfx))\bigg],
\end{align*}
Substituting into \eqref{eq:csmsimp} we get
\begin{align}
    L_{CSM}(q_{\theta}; p) &=& \sum_{d=1}^D \mathbb{E}_{p(\bfx_{< d})} \bigg[ 
   J\left(q_{\theta}(x_d \mid \bfx_{< d});p(x_d \mid \bfx_{< d})\right) + C(\bfx_{< d})
    \bigg] \\
    &=& \sum_{d=1}^D 
    \mathbb{E}_{p(\bfx_{< d})} \bigg[ \mathbb{E}_{p(x_d|\bfx_{<d})} \bigg[  \frac{1}{2}s_{\theta,d}(\bfx)^2 +  \frac{\partial}{\partial x_d} s_{\theta,d}(\bfx)\bigg] + C(\bfx_{< d})
    \bigg] \\
        &=& \sum_{d=1}^D 
    \mathbb{E}_{p(\bfx_{< d})} \bigg[ \mathbb{E}_{p(x_d|\bfx_{<d})} \bigg[  \frac{1}{2}s_{\theta,d}(\bfx)^2 +  \frac{\partial}{\partial x_d} s_{\theta,d}(\bfx)\bigg]
    \bigg] + C
\end{align}
\end{proof}

\subsection[Proofs]{Proof of Theorem~\ref{thm:bijection_family}}
\label{proof:bijection_family}
\begin{proof}
Let $\mathcal{Q}$ be the set of joint distributions that satisfies the condition in Theorem \ref{thm:bijection_family}. Let $f$ be defined as:
\begin{align*}
    f:\text{AR-CSM}&\to \mathcal{Q}\\
    f: \hat{s}(\bfx)=(\hat{s}_1(x_1),...,\hat{s}_{D}(\bfx_{<D}, x_D)) &\mapsto q(\bfx):=\prod_{d=1}^{D}\frac{e^{\mathcal{E}_{d}(\bfx_{< d}, x_d)}}{Z_{d}(\bfx_{< d})}
\end{align*}
\textbf{Surjectivity}\\
Given $q\in \mathcal{Q}$, from the chain rule, we have 
\begin{equation}
    q(\bfx)=\prod_{d=1}^{D}q(x_d|\bfx_{<d}).
\end{equation}
By assumption, we have $\frac{\partial}{\partial x_d}\log q(x_d|\bfx_{<d})$ exists for all $d$.
Define  $\hat{s}_d(\bfx_{< d},x_d)=\frac{\partial}{\partial x_d}\log q(x_d|\bfx_{<d})$ for each $d$. We can check that $\mathcal{E}_d(\bfx_{< d}, x_d)=\log q(x_d|\bfx_{<d})+C$, where $C$ is a constant. We also have $Z_d(\bfx_{< d})=e^{C}\int q(x_d|\bfx_{<d})d x_d=e^{C}$ exists. Thus, $s(\bfx)\in \text{AR-CSM}$. On the other hand, we have
\begin{align}
    f(s(\bfx)) &=\prod_{d=1}^{D}\frac{e^{\mathcal{E}_d(\bfx_{< d}, x_d)}}{Z_d(\bfx_{< d})}\\
    &=\prod_{d=1}^{D}\frac{e^{C}\log q(x_d|\bfx_{<d})}{e^{C}}\\
    &=\prod_{d=1}^{D} \log q(x_d|\bfx_{<d})\\
    &=q(\bfx)
    \label{eq:injective-product}
\end{align}
Thus $s(\bfx)\in \text{AR-CSM}$ is a pre-image of $q(\bfx)$ and $f$ is surjective.\\

\noindent\textbf{Injectivity}\\
Given $q\in \mathcal{Q}$, assume there exist $\hat{s}_1(\bfx), \hat{s}_2(\bfx)\in \text{AR-CSM}$ such that $f(\hat{s}_1(\bfx))=f(\hat{s}_2(\bfx))=q(\bfx)$, we have 
\begin{align*}
    q(\bfx)&=f(\hat{s}_i(\bfx))\\
    \prod_{d=1}^{D}q(x_d|\bfx_{<d})&=\prod_{d=1}^{D}\frac{e^{\mathcal{E}_{i,d}(\bfx_{< d},x_d)}}{Z_{i,d}(\bfx_{< d})},
\end{align*}
where $i=1,2$. 
\begin{lemma}
For any $d=1,...,D$, we have $q(x_d|\bfx_{<d})=\frac{e^{\mathcal{E}_{i,d}(\bfx_{< d},x_d)}}{Z_{i,d}(\bfx_{< d})}$.
\label{lemma:injective1}
\end{lemma}
\begin{proof}
Let's prove this argument using induction on $d$.\\
i) When $d=1$, integrate \eqref{eq:injective-product} \textit{w.r.t.} $x_D,...,x_2$ sequentially, we have
\begin{align*}
    \int...\int \prod_{d=1}^{D}q(x_d|\bfx_{<d})dx_D...dx_2&=\int...\int\prod_{d=1}^{D}\frac{e^{\mathcal{E}_{i,d}(\bfx_{< d},x_d)}}{Z_{i,d}(\bfx_{< d})}dx_D...dx_2\\
    q(x_1)&=\frac{e^{\mathcal{E}_{i,d}(x_1)}}{Z_{i,d}(x_1)}
\end{align*}
Thus, the condition holds when $d=1$.\\
ii) Assume the condition holds for any $d\le k$, that is $q(x_d|\bfx_{<d})=\frac{e^{\mathcal{E}_{i,d}(\bfx_{<d},x_d)}}{Z_{i,d}(\bfx_{<d})}$ for any $d\le k$. This implies
\begin{align*}
    q(x_1,...,x_k)&=\prod_{d=1}^{k}q(x_d|\bfx_{<d})=\prod_{d=1}^{k}\frac{e^{\mathcal{E}_{i,d}(\bfx_{<d},x_d)}}{Z_{i,d}(\bfx_{<d})}
\end{align*}
Similarly, integrating \eqref{eq:injective-product} \textit{w.r.t.} $x_D,...,x_{k+2}$ sequentially will give us
\begin{align*}
    q(x_1,...,x_{k+1})&=\prod_{d=1}^{k+1}\frac{e^{\mathcal{E}_{i,d}(\bfx_{<d},x_d)}}{Z_{i,d}(\bfx_{<d})}.
\end{align*}
Plugging in $q(x_1,...,x_k)=\prod_{d=1}^{k}\frac{e^{\mathcal{E}_{i,d}(\bfx_{<d},x_d)}}{Z_{i,d}(\bfx_{<d})}$ and use the fact that $q(x_1,...,x_k)\neq 0$ (since $q$ has support equals to the entire space by assumption), we obtain $q(x_{k+1}|\bfx_{<k+1})=\frac{e^{\mathcal{E}_{i,k+1}(\bfx_{<k+1}, x_{k+1})}}{Z_{i,k+1}(\bfx_{<k+1})}$. Thus, the hypothesis holds when $d=k+1$.\\
iii) By induction hypothesis, the condition holds for all $d$.
\end{proof}
From \lemref{lemma:injective1}, we have 
\begin{align*}
    q(x_d|\bfx_{<d})&=\frac{e^{\mathcal{E}_{i,d}(\bfx_{< d})}}{Z_{i,d}(\bfx_{< d})}\\
    \log q(x_d|\bfx_{<d})&=\mathcal{E}_{i,d}(\bfx_{< d},x_d) - \log Z_{i,d}(\bfx_{< d}).
\end{align*}
Taking the derivative \textit{w.r.t} $x_d$ on both sides, since $\log Z_{i,d}(\bfx_{< d},x_d)$ does not depend on $x_d$, we conclude that
\begin{align*}
    \frac{\partial}{\partial x_d}\log q(x_d|\bfx_{<d})&=\frac{\partial}{\partial x_d}\mathcal{E}_{i,d}(\bfx_{< d}, x_d)=\hat{s}_{i,d}(\bfx_{< d}, x_d),
\end{align*}
where $i=1,2$. This implies $\hat{s}_1(\bfx)=\hat{s}_2(\bfx)$, and $f$ is injective.
\end{proof}

\section{Data generation on 2D toy datasets}
\label{app:toy_dataset}
We perform density estimation on a couple two-dimensional synthetic distributions with various shapes and number of modes using our method. 
In \figref{fig:ar-sm-toy}, we visualize the samples drawn from our model. We notice that the trained AR-CSM model can fit multi-modal distributions well.
\begin{figure}[H]
\centering
\includegraphics[width=0.8\textwidth]{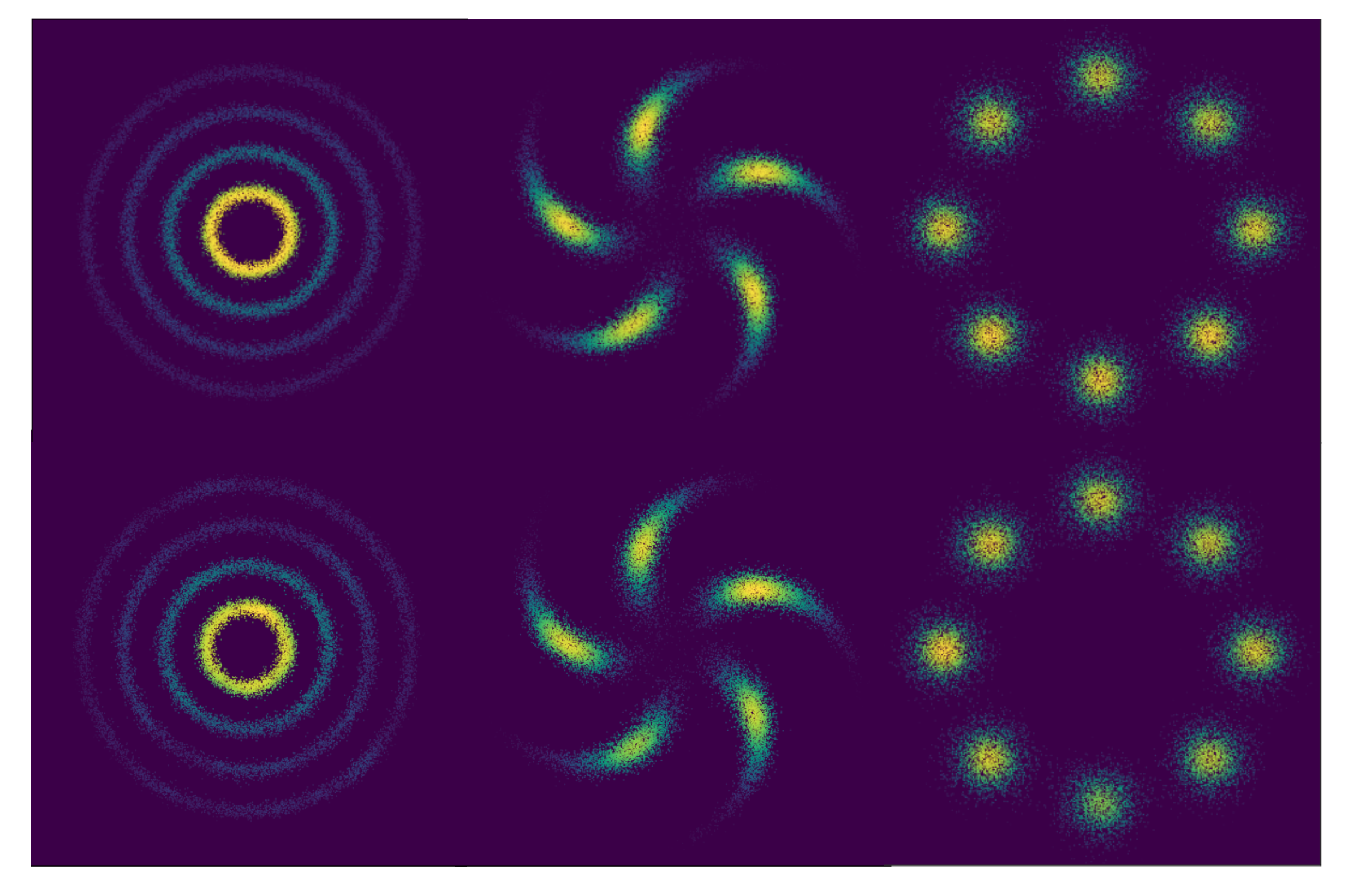}
\caption{Samples from 2D synthetic datasets. The first row: data distribution. The second row: samples from AR-CSM.}
\label{fig:ar-sm-toy}   
\end{figure}

\section{Additional details of AR-CSM experiments}
\label{app:ar-csm}
\subsection{More Samples}
\noindent\textbf{MADE MLE}
\begin{figure}[H]
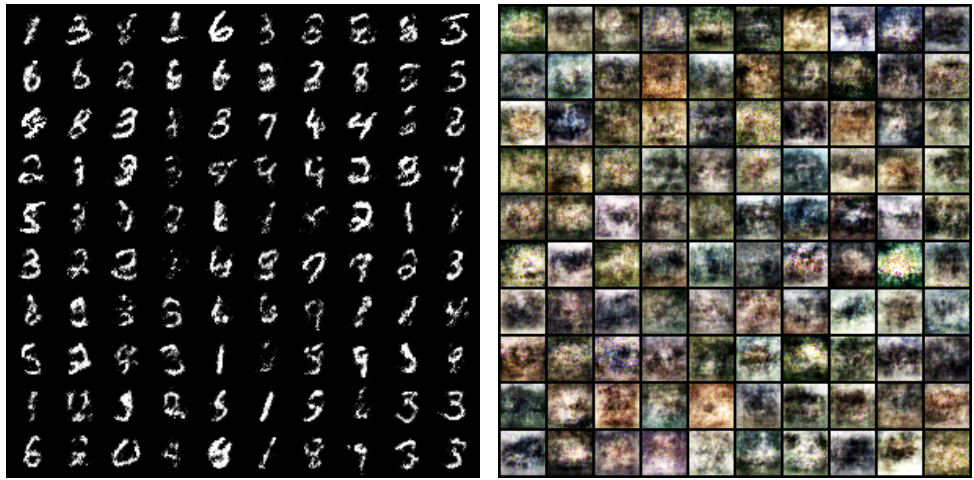

    \centering
    \begin{subfigure}[b]{0.45\textwidth}
        \includegraphics[width=\textwidth]{img/made_direct_samples_new-1.png}
        \caption{MNIST samples.}
    \end{subfigure}
    ~
    \begin{subfigure}[b]{0.45\textwidth}
        \includegraphics[width=\textwidth]{img/made_direct_samples_new.png}
        \caption{CIFAR-10 samples.}
    \end{subfigure}
    \caption{MADE MLE samples.}
\end{figure}

\noindent\textbf{MADE CSM}
\begin{figure}[H]
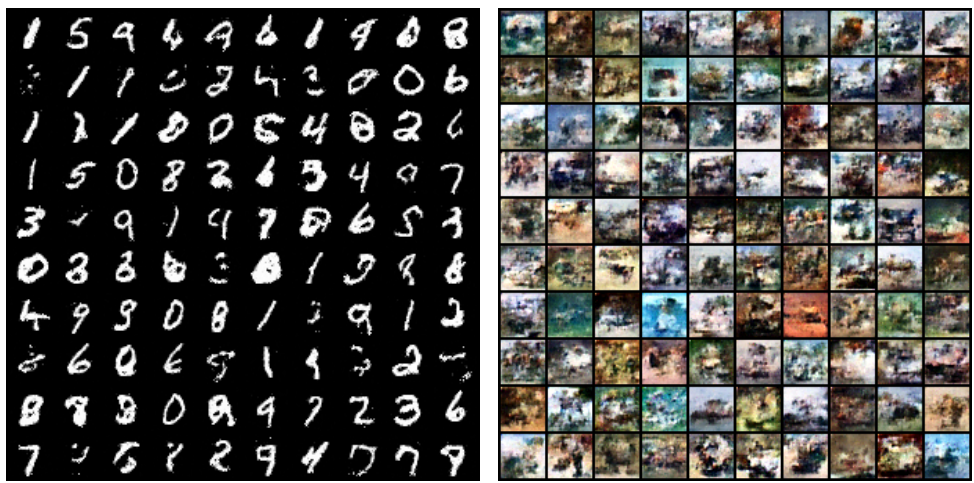

    \centering
    \begin{subfigure}[b]{0.45\textwidth}
        \includegraphics[width=\textwidth]{img/made_sm_mnist_anneal20_2e-05_0.001_-1_m_e.png}
        \caption{MNIST samples.}
    \end{subfigure}
    ~
    \begin{subfigure}[b]{0.45\textwidth}
        \includegraphics[width=\textwidth]{img/made_cifar10_anneal5_8e-05_0.01_-1_m_e.png}
        \caption{CIFAR-10 samples.}
    \end{subfigure}
    \caption{MADE CSM samples.}
\end{figure}

\noindent\textbf{PixelCNN++ MLE}
\begin{figure}[H]
    \centering
    \begin{subfigure}[b]{0.32\textwidth}
        \includegraphics[width=\textwidth]{img/mle_mnist_pixelcnn_sample.png}
        \caption{MNIST samples.}
    \end{subfigure}
    ~
    \begin{subfigure}[b]{0.32\textwidth}
        \includegraphics[width=\textwidth]{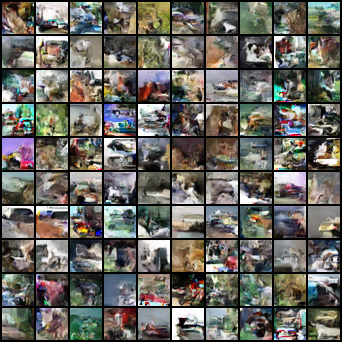}
        \caption{CIFAR-10 samples.}
    \end{subfigure}
    \begin{subfigure}[b]{0.32\textwidth}
        \includegraphics[width=\textwidth]{img/mle_celeba_pixelcnn_sample.png}
        \caption{CelebA samples.}
    \end{subfigure}
    \caption{PixelCNN++ MLE samples.}
\end{figure}

\noindent\textbf{PixelCNN++ CSM}
\begin{figure}[H]
    \centering
    \begin{subfigure}[b]{0.32\textwidth}
        \includegraphics[width=\textwidth]{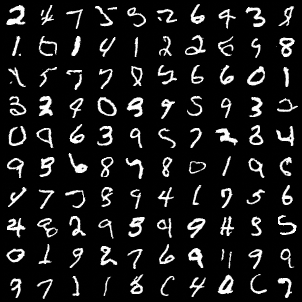}
        \caption{MNIST samples.}
    \end{subfigure}
    ~
    \begin{subfigure}[b]{0.32\textwidth}
        \includegraphics[width=\textwidth]{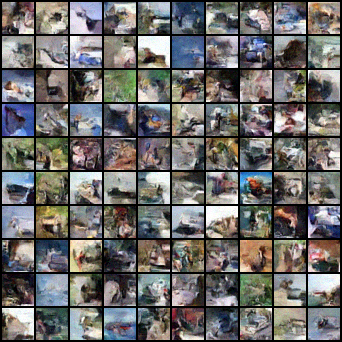}
        \caption{CIFAR-10 samples.}
    \end{subfigure}
    \begin{subfigure}[b]{0.32\textwidth}
        \includegraphics[width=\textwidth]{img/celeba_anneal50_0.0005_0.01_-1_m_e.png}
        \caption{CelebA samples.}
    \end{subfigure}
    \caption{PixelCNN++ CSM samples.}
\end{figure}

\subsection{Noise annealing}
\label{app:noise-annealing}
\begin{figure}[H]
\centering
\includegraphics[width=0.76\textwidth]{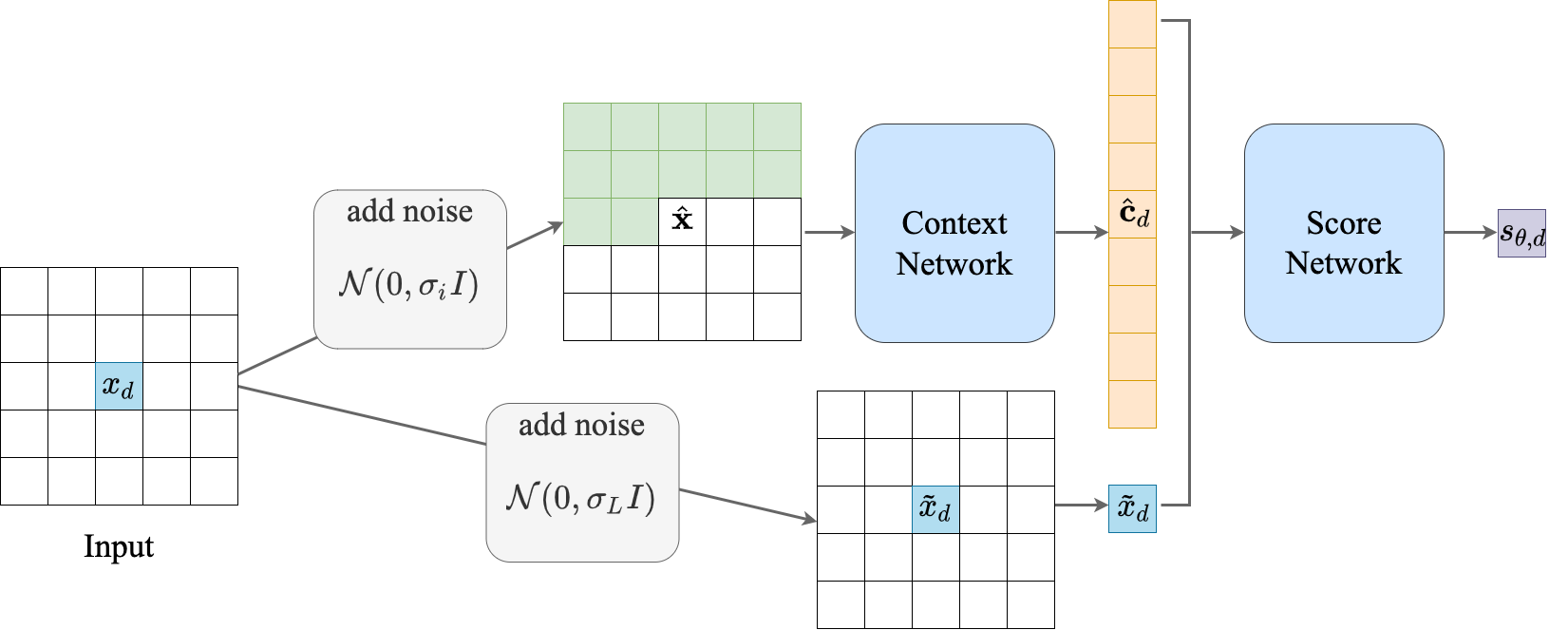}
\caption{Conditional noise annealing at dimension $d$. The context $\hat{c}_d$ only depends on the pixels in front of $\hat{x}_d$ in $\hat{\bfx}$ (\ie the green ones in $\hat{\bfx}$).}
\label{fig:app_anneal_architecture}
\end{figure}

Training score-based generative modeling has been a challenging problem due to the manifold hypothesis and the existence of low data density regions in the data distribution. \cite{song2019generative} shows the efficacy of noise annealing while addressing the above challenges. Based on their arguments, we adopt a noise annealing scheme for training one dimensional score matching. More specifically, we choose a positive geometric sequence $\{\sigma_i\}_{i=1}^{L}$ that satisfies $\frac{\sigma_1}{\sigma_2}=...=\frac{\sigma_{L-1}}{\sigma_L}>1$ to be our noise levels. 
Since at each dimension, we perform score matching \textit{w.r.t.} $x_{d}$ given $\bfx_{<d}$, the previous challenges apply to the one dimensional distribution $p(x_d|\bfx_{<d})$. We thus propose to perform noise annealing only on the scalar $x_d$.
This process requires us to deal with $x_d$ and $\bfx_{<d}$ separately.
For convenience, let us denote the input for the autoregressive model (context network) as $\hat{\bfx}$, and the scalar pixel that will be concatenated with the context vector as $\Tilde{x}_d$.
We decompose the training process into $L$ stages. At stage $i$, we choose $\sigma_i$ to be the noise level for $\Tilde{x}_d$ and use the perturbed noise distribution $p_{\sigma_i}(\Tilde{x}_d|x_d)=\mathcal{N}(\Tilde{x}_d|x_d, \sigma^2_i)$ to obtain $\Tilde{x}_d$. We use a shared noise level $\hat{\sigma}$ among all stages for $\hat{\bfx}$ and use a perturbed noise distribution $p_{\hat{\sigma}}(\hat{\bfx}|\bfx)=\mathcal{N}(\hat{\bfx}, |\bfx, \hat{\sigma}^2I_D)$. 
We feed $\hat{\bfx}$ to the context network to obtain the context vector $\hat{\bfc}_{d}$ and concatenate it with $\Tilde{x}_d$ to obtain $\bfh_d=[\hat{\bfc}_{d}, \Tilde{x}_d]$, which is then fed into the score network to obtain conditional scores $s_{\theta, d}(\bfx)$ (see \figref{fig:app_anneal_architecture}).
At each stage, we train the network until convergence before moving on to the next stage. We denote the learned data distribution and the model distribution at stage $i$ for the $d$-th dimension as $p_{\sigma_i}(\Tilde{x}_d|\hat{\bfx}_{<d})$ and $q_{\theta, \sigma_i}(\Tilde{x}_d|\hat{\bfx}_{<d})$ respectively. As the perturbed noise for $\Tilde{x}_d$ gradually decreases \textit{w.r.t.} the stages, we call this process \textit{conditional noise annealing}. For consistency, we want the distribution of $\hat{\bfx}$ to match the final state distribution of $\Tilde{\bfx}=(\Tilde{x}_1,...,\Tilde{x}_{D})$, we thus choose $\hat{\sigma}=\sigma_{D}$ as the perturbed distribution for $\hat{\bfx}$ among all the stages.

\subsection{Inference with annealed autoregressive Langevin dynamics}
\label{app:inference_langevin_dynamics}
To sample from the model, we can sample each dimension sequentially using one dimensional Langevin dynamics. For the $d$-th dimension, given a fixed step size $\eps > 0$ and an initial value $\Tilde{x}_d^{[0]}$ drawn from a prior distribution $\pi(\bfx)$ (\eg, a standard normal distribution $\mathcal{N}(0,1)$), the one dimensional Langevin method recursively computes the following based on the already sampled previous pixels ${\bfx}^{[T]}_{<d}$
\begin{equation}
    {x}_d^{[t]} = {x}_d^{[t-1]} + \frac {\eps}{2} \frac{\partial}{\partial x_d^{[t-1]}} \log p({x}_d^{[t-1]} | {\bfx}_{<d}^{[T]}) + \sqrt{\eps}\bfz_t,\; t=1,...,T,
\label{eq:ar-langevin}
\end{equation}
where $\bfz_t\sim \mathcal{N}(0,1)$.
When $\eps\to 0$ and $T\to \infty$, the distribution of $x_d^{[T]}$ matches $p({x}_d | {\bfx}_{<d})$, in which case $x_d^{[T]}$ 
is an exact sample from $p({x}_d | {\bfx}_{<d})$ under some regularity conditions~\cite{welling2011bayesian}.
Similar as \cite{song2019generative}, we can use annealed Langevin dynamics to speed up the mixing speed of one dimensional Langevin dynamics. 
Let $\eps_0>0$ be a prespecified constant scalar, we decompose the sampling process into $L$ stages for each $d$. 
At stage $i$, we run autoregressive Langavin dynamics to sample from $p_{\sigma_i}(\Tilde{x}_d|\hat{\bfx}_{<d})$ using the model $q_{\theta,\sigma_i}(\Tilde{x}_d|\hat{\bfx}_{<d})$ learned at the $i$-th stage of the training process. We define the anneal Langevin dynamics update rule as
\begin{equation}
    \Tilde{x}_{d}^{[t]} = \Tilde{x}_{d}^{[t-1]} + \frac {\eps}{2} \frac{\partial}{\partial \Tilde{x}_{d}^{[t-1]}} \log q_{\theta,\sigma_i}(\Tilde{x}_{d}^{[t-1]} | \hat{\bfx}_{<d}^{[T]}) + \sqrt{\eps}\bfz_t,\; t=1,...,T,
\label{eq:ar-langevin-anneal}
\end{equation}
where $\bfz_t\sim \mathcal{N}(0,1)$. We choose the step size $\eps=\eps_0\cdot\frac{\sigma^2_i}{\sigma^2_L}$ for the same reasoning as discussed in \cite{song2019generative}.
At stage $i>1$, we set the initial state $\Tilde{x}_{d}^{[0]}$ to be the final samples of the previous simulation at stage $i-1$; and at stage one, we set the initial value $\Tilde{x}_1^{[0]}$ to be random samples drawn from the prior distribution $\mathcal{N}(0, 1)$. For each dimension $d$, we start from stage one, repeat the anneal Langevin sampling process for $\Tilde{x}_d$ until we reach stage $L$, in which case we have sampled the $d$-th component from our model. Compared to Langevin dynamics performed on a high dimensional space, one dimensional Langevin dynamics is shown to be able to converge faster under certain regularity conditions~\cite{roberts1996exponential}.

\subsection{Setup}
For CelebA, we follow a similar setup as 
\cite{song2019generative}: we first center-crop the images to $140 \times 140$ and then resize them to $32 \times 32$. All images are rescaled so that pixel values are located between $-1$ and $1$. We choose $L = 10$ different noise levels for $\{\sigma_i\}_{i=1}^{L}$. For MNIST, we use $\sigma_1 = 1.0$ and $\sigma_L = 0.04$, and $\sigma_1 = 0.2$ and $\sigma_L = 0.04$ are used for CIFAR-10 and CelebA. We notice that for the used image data, due to the rescaling, a Gaussian noise with $\sigma = 0.04$ is almost indistinguishable to human eyes. 
During sampling, we find $T = 20$ for MNIST and $ T=10$ for CIFAR-10 and CelebA work reasonably well for anneal autoregressive Langevin dynamics in practice. We select two existing autoregressive models, MADE~\cite{germain2015made} and PixelCNN++~\cite{salimans2017pixelcnn++}, as the architectures for our autoregressive context network (AR-CN). For all the experiments, we use a shallow fully connected network as the architecture for the conditional score network (CSN). The amount of parameters for this shallow fully connected network is almost negligible compared to the autoregressive context network.
We train the models for 200 epochs in total, using Adam optimizer with learning rate $0.0002$.

\section{Additional details of VAE experiments}
\label{app:vae_extra}
\subsection{Background}
Given a latent variable model $p(\bfx, \bfz)$ where $\bfx$ is the observed variable and $\bfz$ is the latent variable, a VAE contains the following two parts: i) an encoder $q_{\phi}(\bfz | \bfx)$ that models the conditional distribution of the latent variable given the observed data; and ii) a decoder $p_{\theta}(\bfx|\bfz)$ that models the posterior distribution of the latent variable.
In general, a VAE is trained by maximizing the evidence lower bound (ELBO):
\begin{align}
    \mathbb{E}_{p_{\text{data}}(\bfx)}[\mathbb{E}_{q_{\phi}(\bfz|\bfx)}\log p_{\theta}(\bfx|\bfz)p(\bfz)-\mathbb{E}_{q_{\phi}(\bfz|\bfx)}\log q_{\phi}(\bfz|\bfx)].
\label{eq:app_elbo}
\end{align}
We refer to this traditional training method as "ELBO" throughout the discussion. In ELBO, $q_{\phi}(\bfz|\bfx)$ is often chosen to be a simple distribution such that $H(q_{\phi}(\cdot|\bfx))\triangleq-\mathbb{E}_{q_{\phi}(\bfz|\bfx)}[\log q_{\phi}(\bfz|\bfx)]$ is tractable, which constraints the flexibility of an encoder. 
\subsection{Training VAEs with implicit encoders}
Instead of parameterizing $q_{\phi}(\bfz|\bfx)$ directly as a normalized density function, we can parameterize the encoder using an implicit distribution, which removes the above constraints imposed on ELBO. We call such encoder an implicit encoder.
Denote $H_{d}(q_\phi(\cdot|\bfz_{<d}, \bfx))\triangleq -\mathbb{E}_{q_{\phi}(z_d|\bfz_{<d}, \bfx)}[\log q_{\phi}(z_d|\bfz_{<d},\bfx)]$, using the chain rule of entropy, we have
\begin{align}
H(q_\phi(\cdot|\bfx)))=
    -\mathbb{E}_{q_{\phi}(\bfz|\bfx)}[\log q_{\phi}(\bfz|\bfx)]&=
    \sum_{d=1}^{D} H_{d}(q_{\phi}(\cdot|\bfz_{<d}, \bfx)).
\end{align}
Suppose $z_d\sim q_{\phi}(z_d| \bfz_{<d}, \bfx)$ can be parameterized as $z_d = h_{\phi, d}(\eps_d, \bfz_{<d}, \bfx)$, where $\eps_d$ is a simple one dimensional random variable independent of $\phi$ (\textit{i.e.} a standard normal) and $h_{\phi, d}$ is a deterministic mapping depending on $\phi$ at dimension $d$. By plugging in $\bfz_{<d}$ into $h_{\phi, d}$ and using $z_d = h_{\phi, d}(\eps_d, \bfz_{<d}, \bfx)$ recursively,
we can show that $z_d$ can be reparametrized as $g_{\phi, d}(\bfe_{\le d},\bfx)$, which is a deterministic mapping depending on $\phi$. This provides the following equality for the gradient of $H_{d}(q_{\phi}(\cdot|\bfx))$ \textit{w.r.t.} $\phi$
\begin{align}
    \nabla_{\phi}H_{d}(q_\phi(\cdot|\bfz_{<d}, \bfx))&\triangleq -\nabla_{\phi} \mathbb{E}_{q_{\phi}(z_d|\bfz_{<d}, \bfx)}[\log q_{\phi}(z_d|\bfz_{<d},\bfx)]\\
    &=-\mathbb{E}_{p(\bfe_d)}[\frac{\partial}{\partial z_d}\log q_{\phi}(z_d|\bfz_{<d},\bfx)|_{z_d=g_{\phi, d}(\bfe_{\le d}, \bfx)}\nabla_{\phi}g_{\phi, d}(\bfe_{\le d}, \bfx)]
\label{eq:vae_reparameterization}
\end{align}
(See \appref{app:vae_extra_reparameterization}). This implies
\begin{align*}
    \nabla_{\phi}H(q_\phi(\cdot|\bfx)))
    &=-\sum_{d=1}^{D}\mathbb{E}_{p(\bfe_{<d})}\mathbb{E}_{p(\bfe_d)}\bigg[\frac{\partial}{\partial z_d}\log q_{\phi}(z_d|\bfz_{<d}, \bfx)|_{z_d=g_{\phi, d}(\bfe_{\le d}, \bfx)}\nabla_{\phi}g_{\phi, d}(\bfe_{\le d}, \bfx)\bigg],
\end{align*}
Besides the aforementioned encoder and decoder, to train a VAE model using an implicit encoder, we introduce a third model: an AR-CSM model with parameter $\Tilde{\phi}$ denoted as $s_{\Tilde{\phi}}(\bfz|\bfx)$ that is used to approximate the conditional score $s_{{\phi}, d}(\bfz| \bfx)\triangleq\frac{\partial}{\partial z_d} \log q_{\phi}(z_d|\bfz_{< d},\bfx)$ of the implicit encoder.
During training, we draw i.i.d. samples $\{\bfz^{(1)},...,\bfz^{(N)}\}$ from the implicit encoder $q_{\phi}(\bfz|\bfx)$ and use these samples to train $s_{\Tilde{\phi}}(\bfz|\bfx)$ to approximate the conditional scores of the encoder using CSM. After $s_{\Tilde{\phi}}(\bfz|\bfx)$ is updated, we use the $d$-th component of $s_{\Tilde{\phi}}(\bfz|\bfx)$, the approximation of $\frac{\partial}{\partial z_d}\log  q_{\phi}(z_d|\bfz_{<d}, \bfx)$, as the substitution for $\frac{\partial}{\partial z_d}\log  q_{\phi}(z_d|\bfz_{<d}, \bfx)$ in \eqref{eq:vae_reparameterization}. To compute $\nabla_{\phi}H(q_\phi(\cdot|\bfx))$, we can detach the approximated conditional score $s_{\Tilde{\phi}}(\bfz|\bfx)$ so that the gradient of $H(q_\phi(\cdot|\bfx))$ could be approximated properly using PyTorch backpropagation. This provides us with a way to evaluate the gradient of \equref{eq:elbo} \textit{w.r.t.} $\phi$, which can be used to update the implicit encoder $q_{\phi}(\bfz|\bfx)$.

\subsection{Reparameterization}
\label{app:vae_extra_reparameterization}
\begin{align}
    \nabla_{\phi}H_{d}(q_\phi(\cdot|\bfz_{<d}, \bfx))&\triangleq -\nabla_{\phi} \mathbb{E}_{q_{\phi}(z_d|\bfz_{<d}, \bfx)}[\log q_{\phi}(z_d|\bfz_{<d},\bfx)]\\
    &=-\nabla_{\phi} \mathbb{E}_{p(\bfe_d)}[\log q_{\phi}(g_{\phi, d}(\bfe_{\le d}, \bfx))]\\
    &=-\mathbb{E}_{p(\bfe_d)}[\nabla_{\phi}\log q_{\phi}(g_{\phi, d}(\bfe_{\le d}, \bfx))]\\
    &=-\mathbb{E}_{p(\bfe_d)}[\frac{\partial}{\partial z_d}\log q_{\phi}(z_d|\bfz_{<d}, \bfx)|_{z_d=g_{\phi, d}(\bfe_{\le d}, \bfx)}\nabla_{\phi}g_{\phi, d}(\bfe_{\le d}, \bfx)].
\label{eq:vae_reparameterization_full}
\end{align}

\subsection{Setup}
\label{app:vae_setup}
For CelebA, we follow the setup in \cite{song2019sliced}. We first center-crop all images to a patch of $140 \times 140$, and then resize the image size to $64\times 64$. For MNIST experiments, we use RMSProp optimizer with a learning rate of 0.001 for all methods except for the CSM experiments where we use learning rate of 0.0002 for the score estimator. On CelebA,  we use RMSProp optimizer with a learning rate of 0.0001 for all methods except for the CSM experiments where we use a learning rate of 0.0002 for the score estimator.

\newpage
\section{VAE with implicit encoders}
\label{app:vae_samples}
\textbf{CSM}
\\
\begin{figure}[H]
    \centering
    \begin{subfigure}[b]{0.32\textwidth}
        \includegraphics[width=\textwidth]{img/ar_sm_8.png}
    \end{subfigure}
    \begin{subfigure}[b]{0.32\textwidth}
        \includegraphics[width=\textwidth]{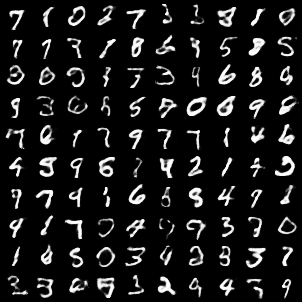}
    \end{subfigure}
    \begin{subfigure}[b]{0.32\textwidth}
        \includegraphics[width=\textwidth]{img/ar_sm_celeba.png}
    \end{subfigure}
    \caption{From left to right: VAE CSM MNIST samples with latent dimension  8, VAE CSM MNIST samples with latent dimension  16, VAE CSM CelebA samples with latent dimension 32.}
    \label{fig:vae_elbo_samples}
\end{figure}
\textbf{ELBO}
\\
\begin{figure}[H]
    \centering
    \begin{subfigure}[b]{0.32\textwidth}
        \includegraphics[width=\textwidth]{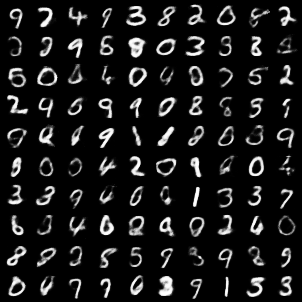}
    \end{subfigure}
    \begin{subfigure}[b]{0.32\textwidth}
        \includegraphics[width=\textwidth]{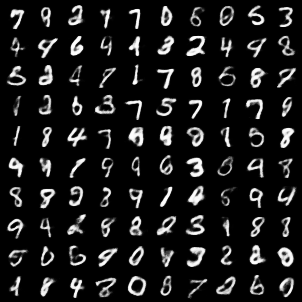}
    \end{subfigure}
    \begin{subfigure}[b]{0.32\textwidth}
        \includegraphics[width=\textwidth]{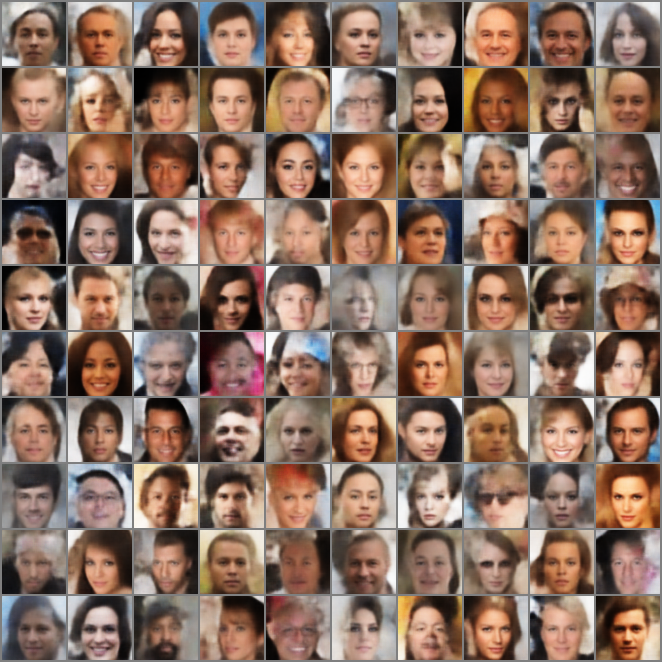}
    \end{subfigure}
    \caption{From left to right: VAE ELBO MNIST samples with latent dimension  8, VAE ELBO MNIST samples with latent dimension  16, VAE ELBO CelebA samples with latent dimension 32.}
    \label{fig:vae_elbo_samples}
\end{figure}
\textbf{Stein}
\\
\begin{figure}[H]
    \centering
    \begin{subfigure}[b]{0.32\textwidth}
        \includegraphics[width=\textwidth]{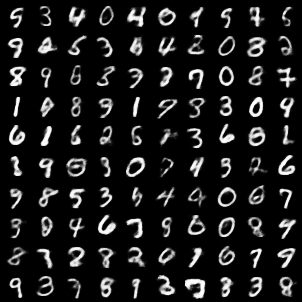}
    \end{subfigure}
    \begin{subfigure}[b]{0.32\textwidth}
        \includegraphics[width=\textwidth]{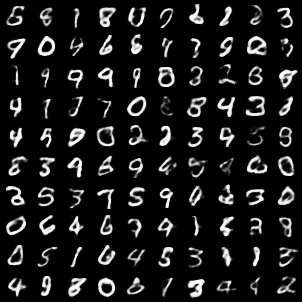}
    \end{subfigure}
    \begin{subfigure}[b]{0.32\textwidth}
        \includegraphics[width=\textwidth]{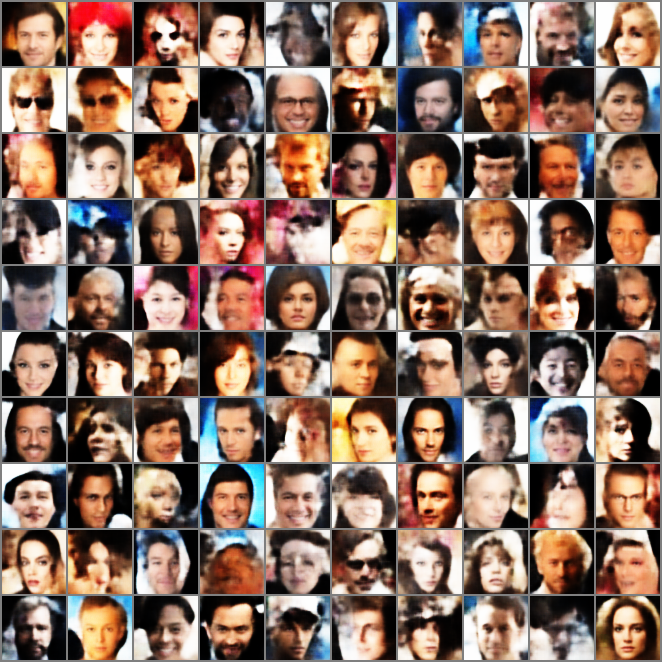}
    \end{subfigure}
    \caption{From left to right: VAE Stein MNIST samples with latent dimension  8, VAE Stein MNIST samples with latent dimension  16, VAE Stein CelebA samples with latent dimension 32.}
    \label{fig:vae_stein_samples}
\end{figure}
\textbf{Spectral}
\\
\begin{figure}[H]
    \centering
    \begin{subfigure}[b]{0.32\textwidth}
        \includegraphics[width=\textwidth]{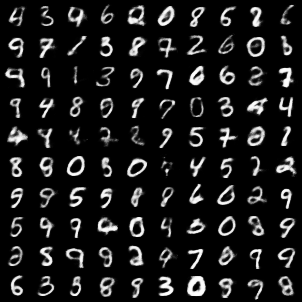}
    \end{subfigure}
    \begin{subfigure}[b]{0.32\textwidth}
        \includegraphics[width=\textwidth]{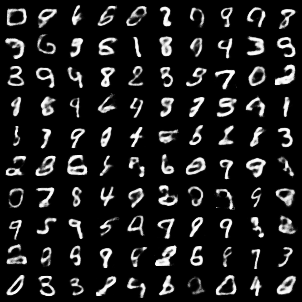}
    \end{subfigure}
    \begin{subfigure}[b]{0.32\textwidth}
        \includegraphics[width=\textwidth]{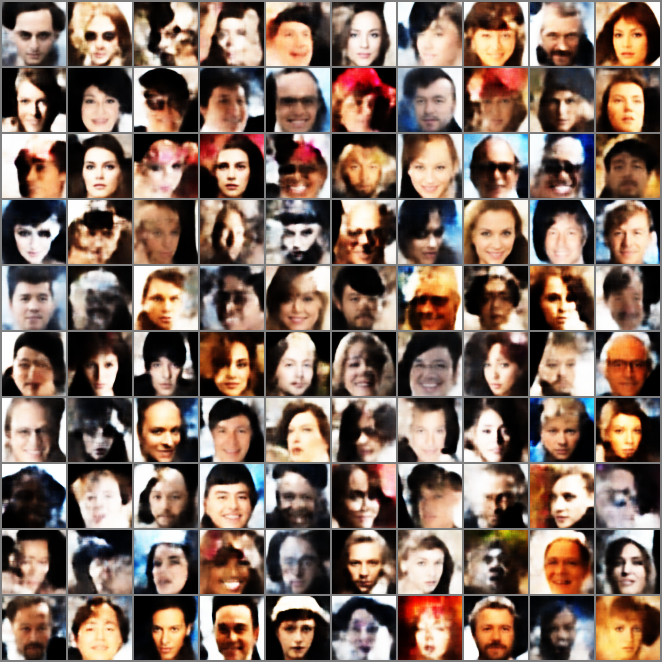}
    \end{subfigure}
    \caption{From left to right: VAE Spectral MNIST samples with latent dimension  8, VAE Spectral MNIST samples with latent dimension  16, VAE Spectral CelebA samples with latent dimension 32.}
    \label{fig:vae_spectral_samples}
\end{figure}
\textbf{SSM-AR}
\\
\begin{figure}[H]
    \centering
    \begin{subfigure}[b]{0.32\textwidth}
        \includegraphics[width=\textwidth]{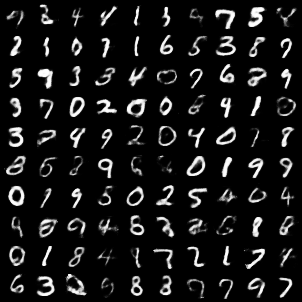}
    \end{subfigure}
    \begin{subfigure}[b]{0.32\textwidth}
        \includegraphics[width=\textwidth]{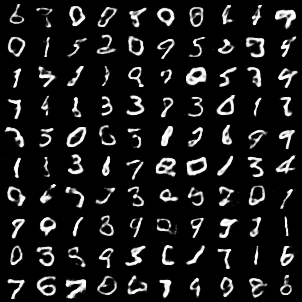}
    \end{subfigure}
    \begin{subfigure}[b]{0.32\textwidth}
        \includegraphics[width=\textwidth]{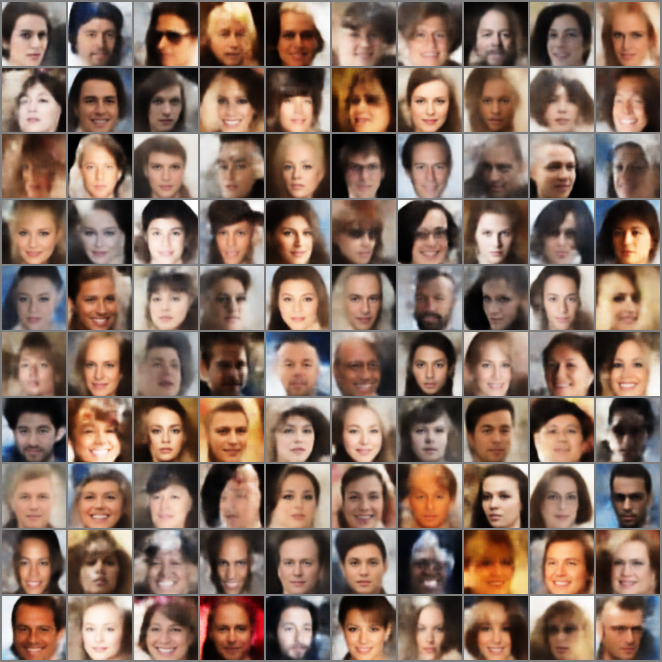}
    \end{subfigure}
    \caption{From left to right: VAE SSM-AR MNIST samples with latent dimension  8, VAE SSM-AR MNIST samples with latent dimension  16, VAE SSM-AR CelebA samples with latent dimension 32.}
    \label{fig:vae_ssm_ar_samples}
\end{figure}
\textbf{SSM}
\begin{figure}[H]
    \centering
    \begin{subfigure}[b]{0.32\textwidth}
        \includegraphics[width=\textwidth]{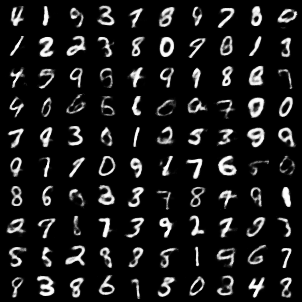}
    \end{subfigure}
    \begin{subfigure}[b]{0.32\textwidth}
        \includegraphics[width=\textwidth]{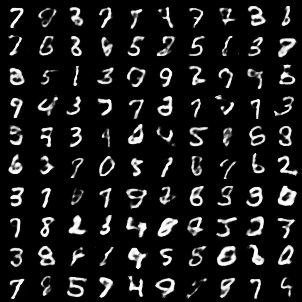}
    \end{subfigure}
    \begin{subfigure}[b]{0.32\textwidth}
        \includegraphics[width=\textwidth]{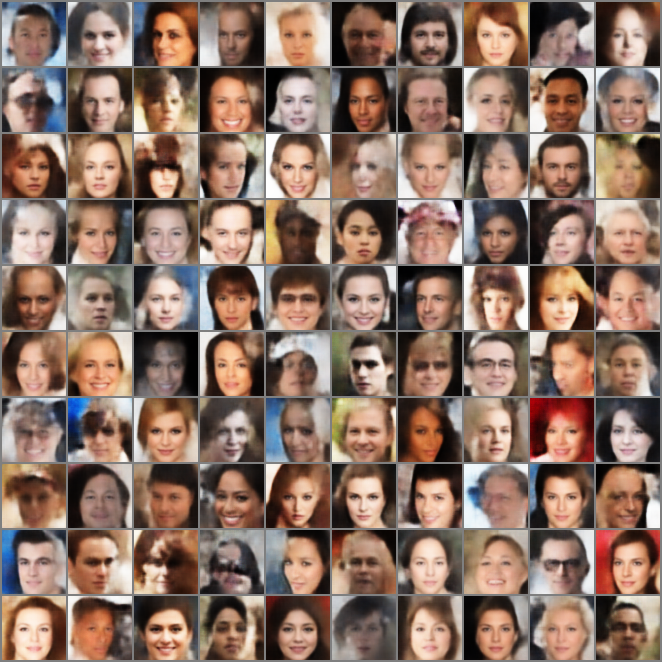}
    \end{subfigure}
    \caption{From left to right: VAE SSM MNIST samples with latent dimension  8, VAE SSM MNIST samples with latent dimension 16, VAE SSM CelebA samples with latent dimension 32.}
    \label{fig:vae_ssm_samples}
\end{figure}

%% file: main.bbl
\begin{thebibliography}{10}

\bibitem{dawid2014theory}
A.~P. Dawid and M.~Musio.
\newblock Theory and applications of proper scoring rules.
\newblock {\em Metron}, 72(2):169--183, 2014.

\bibitem{dinh2016density}
L.~Dinh, J.~Sohl-Dickstein, and S.~Bengio.
\newblock Density estimation using real nvp.
\newblock {\em arXiv preprint arXiv:1605.08803}, 2016.

\bibitem{du2019implicit}
Y.~Du and I.~Mordatch.
\newblock Implicit generation and generalization in energy-based models.
\newblock {\em arXiv preprint arXiv:1903.08689}, 2019.

\bibitem{germain2015made}
M.~Germain, K.~Gregor, I.~Murray, and H.~Larochelle.
\newblock Made: Masked autoencoder for distribution estimation.
\newblock In {\em International Conference on Machine Learning}, pages
  881--889, 2015.

\bibitem{goodfellow2014generative}
I.~Goodfellow, J.~Pouget-Abadie, M.~Mirza, B.~Xu, D.~Warde-Farley, S.~Ozair,
  A.~Courville, and Y.~Bengio.
\newblock Generative adversarial nets.
\newblock In {\em Advances in neural information processing systems}, pages
  2672--2680, 2014.

\bibitem{grenander1994representations}
U.~Grenander and M.~I. Miller.
\newblock Representations of knowledge in complex systems.
\newblock {\em Journal of the Royal Statistical Society: Series B
  (Methodological)}, 56(4):549--581, 1994.

\bibitem{heusel2017gans}
M.~Heusel, H.~Ramsauer, T.~Unterthiner, B.~Nessler, and S.~Hochreiter.
\newblock Gans trained by a two time-scale update rule converge to a local nash
  equilibrium.
\newblock In {\em Advances in neural information processing systems}, pages
  6626--6637, 2017.

\bibitem{huszar2017variational}
F.~Husz{\'a}r.
\newblock Variational inference using implicit distributions.
\newblock {\em arXiv preprint arXiv:1702.08235}, 2017.

\bibitem{hyvarinen2005estimation}
A.~Hyv{\"a}rinen.
\newblock Estimation of non-normalized statistical models by score matching.
\newblock {\em Journal of Machine Learning Research}, 6(Apr):695--709, 2005.

\bibitem{kingma2018glow}
D.~P. Kingma and P.~Dhariwal.
\newblock Glow: Generative flow with invertible 1x1 convolutions.
\newblock In {\em Advances in Neural Information Processing Systems}, pages
  10215--10224, 2018.

\bibitem{kingma2013auto}
D.~P. Kingma and M.~Welling.
\newblock Auto-encoding variational bayes.
\newblock {\em arXiv preprint arXiv:1312.6114}, 2013.

\bibitem{krizhevsky2009learning}
A.~Krizhevsky, G.~Hinton, et~al.
\newblock Learning multiple layers of features from tiny images.
\newblock 2009.

\bibitem{liu2015deep}
Z.~Liu, P.~Luo, X.~Wang, and X.~Tang.
\newblock Deep learning face attributes in the wild.
\newblock In {\em Proceedings of the IEEE international conference on computer
  vision}, pages 3730--3738, 2015.

\bibitem{martens2012estimating}
J.~Martens, I.~Sutskever, and K.~Swersky.
\newblock Estimating the hessian by back-propagating curvature.
\newblock {\em arXiv preprint arXiv:1206.6464}, 2012.

\bibitem{nash2019autoregressive}
C.~Nash and C.~Durkan.
\newblock Autoregressive energy machines.
\newblock {\em arXiv preprint arXiv:1904.05626}, 2019.

\bibitem{neal2001annealed}
R.~M. Neal.
\newblock Annealed importance sampling.
\newblock {\em Statistics and computing}, 11(2):125--139, 2001.

\bibitem{oord2016wavenet}
A.~v.~d. Oord, S.~Dieleman, H.~Zen, K.~Simonyan, O.~Vinyals, A.~Graves,
  N.~Kalchbrenner, A.~Senior, and K.~Kavukcuoglu.
\newblock Wavenet: A generative model for raw audio.
\newblock {\em arXiv preprint arXiv:1609.03499}, 2016.

\bibitem{oord2016pixel}
A.~v.~d. Oord, N.~Kalchbrenner, and K.~Kavukcuoglu.
\newblock Pixel recurrent neural networks.
\newblock {\em arXiv preprint arXiv:1601.06759}, 2016.

\bibitem{parisi1981correlation}
G.~Parisi.
\newblock Correlation functions and computer simulations.
\newblock {\em Nuclear Physics B}, 180(3):378--384, 1981.

\bibitem{roberts1996exponential}
G.~O. Roberts, R.~L. Tweedie, et~al.
\newblock Exponential convergence of langevin distributions and their discrete
  approximations.
\newblock {\em Bernoulli}, 2(4):341--363, 1996.

\bibitem{salimans2017pixelcnn++}
T.~Salimans, A.~Karpathy, X.~Chen, and D.~P. Kingma.
\newblock Pixelcnn++: Improving the pixelcnn with discretized logistic mixture
  likelihood and other modifications.
\newblock {\em arXiv preprint arXiv:1701.05517}, 2017.

\bibitem{saremi2018deep}
S.~Saremi, A.~Mehrjou, B.~Sch{\"o}lkopf, and A.~Hyv{\"a}rinen.
\newblock Deep energy estimator networks.
\newblock {\em arXiv preprint arXiv:1805.08306}, 2018.

\bibitem{shi2018spectral}
J.~Shi, S.~Sun, and J.~Zhu.
\newblock A spectral approach to gradient estimation for implicit
  distributions.
\newblock {\em arXiv preprint arXiv:1806.02925}, 2018.

\bibitem{song2019generative}
Y.~Song and S.~Ermon.
\newblock Generative modeling by estimating gradients of the data distribution.
\newblock In {\em Advances in Neural Information Processing Systems}, pages
  11895--11907, 2019.

\bibitem{song2019sliced}
Y.~Song, S.~Garg, J.~Shi, and S.~Ermon.
\newblock Sliced score matching: {A} scalable approach to density and score
  estimation.
\newblock In {\em Proceedings of the Thirty-Fifth Conference on Uncertainty in
  Artificial Intelligence, {UAI} 2019, Tel Aviv, Israel, July 22-25, 2019},
  page 204, 2019.

\bibitem{stein1981estimation}
C.~M. Stein.
\newblock Estimation of the mean of a multivariate normal distribution.
\newblock {\em The annals of Statistics}, pages 1135--1151, 1981.

\bibitem{van2016conditional}
A.~Van~den Oord, N.~Kalchbrenner, L.~Espeholt, O.~Vinyals, A.~Graves, et~al.
\newblock Conditional image generation with pixelcnn decoders.
\newblock In {\em Advances in neural information processing systems}, pages
  4790--4798, 2016.

\bibitem{vincent2011connection}
P.~Vincent.
\newblock A connection between score matching and denoising autoencoders.
\newblock {\em Neural computation}, 23(7):1661--1674, 2011.

\bibitem{vinyals2019grandmaster}
O.~Vinyals, I.~Babuschkin, W.~M. Czarnecki, M.~Mathieu, A.~Dudzik, J.~Chung,
  D.~H. Choi, R.~Powell, T.~Ewalds, P.~Georgiev, et~al.
\newblock Grandmaster level in starcraft ii using multi-agent reinforcement
  learning.
\newblock {\em Nature}, 575(7782):350--354, 2019.

\bibitem{welling2011bayesian}
M.~Welling and Y.~W. Teh.
\newblock Bayesian learning via stochastic gradient langevin dynamics.
\newblock In {\em Proceedings of the 28th international conference on machine
  learning (ICML-11)}, pages 681--688, 2011.

\bibitem{yu2020training}
L.~Yu, Y.~Song, J.~Song, and S.~Ermon.
\newblock Training deep energy-based models with f-divergence minimization.
\newblock {\em arXiv preprint arXiv:2003.03463}, 2020.

\end{thebibliography}
